%% file: arxiv_v1.tex
\documentclass{article} 
\usepackage{times}

\input{math_commands.tex}

\usepackage{fullpage}

\usepackage{natbib}

\usepackage{hyperref}
\usepackage{url}
\usepackage{graphicx} 
\usepackage{amsmath,mathtools}
\usepackage{amsthm}
\usepackage{color}
\usepackage{amsfonts}       
\usepackage{amssymb}
\usepackage{thmtools,thm-restate}
\usepackage{enumitem}
\makeatletter
\theoremstyle{plain}

\theoremstyle{remark}

\theoremstyle{plain}

\newcommand{\marco}[1]{\textcolor{red}{MM: #1}}
\newcommand{\peter}[1]{\textcolor{blue}{PS: #1}}

\newtheorem{theorem}{Theorem}[section]

\newtheorem{prop}[theorem]{Proposition}
\newtheorem{assumption}[theorem]{Assumption}
\newtheorem{lemma}[theorem]{Lemma}

\def\b0{{0}}

\def\RR{\mathbb{R}}

\def\>{\rangle}

\newcommand{\norm}[1]{\left\lVert#1\right\rVert}

\newcommand{\svmin}[1]{\sigma_{\textrm{min}}\left(#1\right)}

\def\min{\mathop{\rm min}\nolimits}
\def\max{\mathop{\rm max}\nolimits}

\title{Wide Neural Networks Trained with Weight Decay \\Provably Exhibit Neural Collapse}

\author{Arthur Jacot\thanks{Courant Institute of Mathematical Sciences, NYU. Email: \texttt{arthur.jacot@nyu.edu}}, \,\,\,Peter Súkeník\thanks{Institute of Science and Technology Austria. Email: \texttt{peter.sukenik@ista.ac.at}}, \,\,\,Zihan Wang\thanks{Courant Institute of Mathematical Sciences, NYU. Email: \texttt{zw3508@nyu.edu}}, \,\,\,and Marco Mondelli\thanks{Institute of Science and Technology Austria. Email: \texttt{marco.mondelli@ist.ac.at}}}

\begin{document}

\maketitle

\begin{abstract}
Deep neural networks (DNNs) at convergence consistently represent the training data in the last layer via a highly symmetric geometric structure referred to as neural collapse. This empirical evidence has spurred a line of theoretical research aimed at proving the emergence of neural collapse, mostly focusing on the unconstrained features model. Here, the features of the penultimate layer are 
free variables, which makes the model data-agnostic and, hence, puts into question its ability to capture DNN training. Our work addresses the issue, moving away from unconstrained features and studying DNNs that end with at least two linear layers. We first prove generic guarantees on neural collapse that assume \emph{(i)} low training error and balancedness of the linear layers (for within-class variability collapse), and \emph{(ii)} bounded conditioning of the features before the linear part (for orthogonality of class-means, as well as their alignment with weight matrices). We then show that such assumptions hold 
for gradient descent training with weight decay: \emph{(i)} for networks with a wide first layer, we prove low training error and balancedness, and \emph{(ii)} for solutions that are either nearly optimal or stable under large learning rates, we additionally prove the bounded conditioning. Taken together, our results are the first to show neural collapse in the end-to-end training of DNNs. 
\end{abstract}

\section{Introduction}

Among the many possible interpolators that a deep neural network (DNN) can find,  
\citet{papyan2020prevalence} showed a strong bias of gradient-based training 
towards representations with a highly symmetric structure in the penultimate layer, which was dubbed \textit{neural collapse} (NC). In particular, the feature vectors of the training data in the penultimate layer collapse to a single vector per class (NC1); these vectors form orthogonal or simplex equiangular tight frames (NC2), and they are aligned with the last layer's row weight vectors (NC3). The question of why and how neural collapse emerges has been considered by 
a popular line of research, 
see e.g.\ \citet{lu2022neural, wojtowytsch2020emergence} and the discussion in Section~\ref{sec:related_work}. Many of these works focus on a simplified mathematical framework: the unconstrained features model (UFM) \citep{mixon2020neural, han2021neural, zhou2022optimization}, corresponding to the joint optimization over the last layer's weights and the penultimate layer's feature representations, which are treated as free variables. To account for the existence of the training data and of all the layers before the penultimate (i.e., the backbone of the network), some form of regularization on the free features is usually added. A number of papers has proved the optimality of NC in this model \citep{lu2022neural, wojtowytsch2020emergence}, its emergence with gradient-based methods \citep{mixon2020neural, han2021neural} and a benign loss landscape \citep{zhou2022optimization, zhu2021geometric}. However, the major drawback of the UFM lies in its data-agnostic nature: it only acknowledges the presence of training data and backbone through a simple form of regularization (e.g., Frobenius norm or sphere constraint), which is far from being equivalent to end-to-end training. Moving beyond UFM, 
existing results are either only applicable to rather shallow networks (of at most three layers) \citep{kothapalli2024kernel, hong2024beyond} or hold under strong assumptions, such as symmetric quasi-interpolation \citep{xu2023dynamics, rangamani2022neural}, block-structured empirical NTK throughout training \citep{seleznova2023neural}, or geodesic structure of the features across all layers \citep{wang2024progressive}. 

In this paper, we provide the first end-to-end proof of within-class variability collapse (NC1) for a class of networks that end with at least two linear layers. 
Furthermore, we give rather weak sufficient conditions -- either near-optimality or stability under large learning rates -- for solutions to exhibit the orthogonality of class-means (NC2) and the alignment of class means with the last weight matrix (NC3). 
More precisely, our contributions can be summarized as follows:
\begin{itemize}[leftmargin=6mm]
    \item First, we show that within-class variability collapse (i.e., NC1) occurs, 
    as long as the training error is low and the linear layers are approximately balanced, i.e., $\left\Vert W_{\ell+1}^TW_{\ell+1}-W_\ell W_\ell ^T\right\Vert_F$ is small, where $W_{\ell+1}, W_\ell$ are two consecutive weight matrices of the linear head. If, additionally, the conditioning of the linear head mapping is bounded, we bound the conditioning of the matrix of class means in the last layer, as well as that of the last weight matrix. This implies that, as the number of linear layers grows,  
    class means become orthogonal and, furthermore, they align with the last layer's row vectors, which proves both NC2 and NC3.
    \item Next, we show that the sufficient conditions above for NC1 are satisfied by a class of deep 
    networks with a wide first layer, smooth activations and pyramidal topology, after gradient training with weight decay. 
    This provides the first guarantee of the emergence of NC1 for a deep network trained end-to-end via gradient descent. 
    \item We further present two sufficient conditions under which the linear head is guaranteed to be 
    well-conditioned, hence NC1, NC2 and NC3 hold: either the network approaches a global optimum of the $\ell_2$-regularized square loss, or it nearly interpolates the data while being stable under large learning rates. 
    \item Our numerical experiments on various architectures (fully connected, ResNet) and datasets (MNIST, CIFAR) confirm the insights coming from the theory: \emph{(i)} NC2 is more prominent as the depth of the linear head increases, and \emph{(ii)} the final linear layers are balanced at convergence. Furthermore, we show that, as the non-linear part of the network gets deeper, the non-negative layers become less non-linear and more balanced.
\end{itemize}

\section{Related Work} \label{sec:related_work}

\paragraph{Neural collapse.} Since its introduction by \citet{papyan2020prevalence}, neural collapse has been intensively studied, both from a theoretical and practical viewpoint. Practitioners use the NC for a number of applications, including transfer learning, OOD detection and generalization bounds \citep{galanti2022improved, haas2022linking, ben2022nearest, li2022principled, li2023no, zhang2024epa}. On the theoretical front, 
the most widely adopted framework to study the emergence of NC is the unconstrained features model (UFM) \citep{mixon2020neural, fang2021exploring}. 
Under the UFM, the NC has been proved to be optimal with cross-entropy loss \citep{wojtowytsch2020emergence, lu2022neural, kunin2022asymmetric}, MSE loss \citep{zhou2022optimization} and other losses \citep{zhou2022all}. Optimality guarantees for the generalization of NC to the class-imbalanced setting have been provided by \citet{fang2021exploring, thrampoulidis2022imbalance, hong2023neural, dang2024neural}. Besides global optimality, a benign loss landscape around NC solutions has been proved in \citet{zhu2021geometric, ji2021unconstrained, zhou2022optimization}, and the emergence of NC with gradient-based optimization under UFM has been studied by \citet{mixon2020neural, han2021neural, ji2021unconstrained, wang2022linear}. \citet{jiang2023generalized} extend the analysis to large number of classes, \citet{kothapalli2023neural} to graph neural networks, \citet{tirer2022perturbation} generalize UFM with a perturbation to account for its imperfections, while \cite{andriopoulos2024prevalence} generalize UFM and NC to regression problems. A deep neural collapse is theoretically analyzed with deep UFM for the linear case in \citet{dang2023neural, garrod2024unifying}, for the non-linear case with two layers in \citet{tirer2022extended}, and for the deep non-linear case in \citet{sukenik2023deep, sukenik2024neural}, where, notably, the latter work provides the first negative result on deep NC with deep UFM. 

A line of recent work aims at circumventing the data-agnostic nature of UFM, showing the emergence of neural collapse in settings closer to practice. Specifically, \citet{seleznova2023neural} 
assume a block structure in the empirical NTK matrix. Kernels are used by \citet{kothapalli2024kernel} to analyze NC in wide two-layer networks, showing mostly negative results in the NTK regime. \citet{beaglehole2024average} prove the emergence of deep neural collapse using kernel-based layer-wise 
training, 
and they show that NC is an optimal solution of adaptive kernel ridge regression in an over-parametrized regime. Sufficient (but rather strong) conditions for the emergence of NC beyond UFM are provided by \citet{pan2023towards}. \citet{hong2024beyond} analyze two and three layer networks end-to-end, but only provide conditions under which the UFM optimal solutions are feasible. 
\citet{wang2024progressive} consider the NC formation in residual networks, and the model is similar to the perturbed UFM in \citet{tirer2022perturbation}. However, the results crucially assume that the features lie on a geodesic in euclidean space, and proving that this is the case after training the ResNet is an open problem. 
%
 \citet{rangamani2022neural, xu2023dynamics} focus on 
 homogeneous networks trained via gradient-based methods, making the strong assumption of symmetric quasi-interpolation. 
 Then, \citet{rangamani2022neural} do not prove this assumption, and the argument of \cite{xu2023dynamics} requires a regularization different from the one used in practice as well as interpolators with a given norm (whose existence is an open question). 
 We also note that the quasi-interpolation property does not hold in practice exactly, which points to the need of a perturbation analysis. 


\paragraph{Implicit bias.} Our approach 
leverages the (approximate) balancedness of the weights, which plays a central role in the analysis of the training dynamics of linear networks \citep{arora_2018_optimization} and leads to a bias towards low-rank matrices \citep{arora_2019_matrix_factorization,tu_2024+mixeddynamicslinearnetworks}. In the presence of weight decay, the low-rank bias can be made even more explicit \citep{dai_2021_repres_cost_DLN}, and it is reinforced by stochastic gradient descent \citep{wang2023implicit}. Moving towards nonlinear models, 
shallow networks with weight decay exhibit low-rank bias, as described by the 
variational norm \citep{bach2017_F1_norm}, or Barron norm \citep{weinan_2019_barron}, and this bias provably emerges under (modified) GD dynamics \citep{abbe2022merged_staircase,bietti_2022_single_index,lee_2024_low_polynomial}. However, a single hidden layer appears to be insufficient to exhibit NC for most datasets. Moving towards deep nonlinear models, networks with weight decay are also known to exhibit low-rank bias \citep{galanti_2022_sgd_low_rank,jacot2022feature}, which can be related to neural collapse \citep{zangrando2024neural}. The dimensionality and rank of the weights/representations varies between layers, exhibiting a 
bottleneck structure \citep{jacot2022implicit,jacot_2023_bottleneck2,wen2024frequencies}. We remark that existing results apply to the global minima of the $\ell_2$-regularized loss. In contrast, our paper provides rather general sufficient conditions that are then provably satisfied by GD, thus showing that the training algorithm is responsible for neural collapse. 


\section{Balancedness and Interpolation Imply Neural Collapse}



\paragraph{Notation and problem setup.} Given a matrix $A$ of rank $k$, we denote by $A_{i:}$ its $i$-th row, 
by $A_{:i}$ its $i$-th column, by $s_1(A)\ge \cdots \ge s_k(A)$ its singular values in non-increasing order, and by $\kappa(A)$ the ratio $\frac{s_1(A)}{s_k(A)}$ between its largest and smallest non-zero singular values. We denote by $\|A\|_F, \|A\|_{op}$ and $ \svmin{A}$ its Frobenius norm, its operator norm and its smallest singular value, respectively. 

We consider a neural network with $L_1$ non-linear layers with activation function $\sigma:\RR\to\RR$ followed by $L_2$ linear layers. Let $L:=L_1+L_2$ be the total number of layers, $W_\ell\in\RR^{n_{\ell}\times n_{\ell-1}}$ the weight matrix at layer $\ell$, $X\in\RR^{d\times N}$ the training data and $Y\in\RR^{K\times N}$ the labels (for consistency, we set $n_0=d$ and $n_L=K$), where the output dimension $K$ corresponds to the number of classes and $N$ is the number of samples. 
We consider a one-hot encoding, i.e., the rows of $Y$ are elements of the canonical basis. 
Let $Z_\ell\in\RR^{n_\ell\times N}$ be the output of layer $\ell$, given by 
\begin{align}\label{eq:F_l}
    Z_\ell = 
    \begin{cases}
	X & \ell=0,\\
	\sigma \big( W_\ell Z_{\ell-1}  \big) & \ell\in [L_1],\\
	W_\ell Z_{\ell-1}  & \ell\in \{L_1+1, \ldots, L_1+L_2\},
    \end{cases}
\end{align}
where the activation function $\sigma$ is applied componentwise and, given an integer $n$, we use the shorthand $[n]:=\{1, \ldots, n\}$.  
We write $\theta$ for the vector obtained by concatenating all parameters $\{W_i\}_{i\in[L]}$ and $W_{m:\ell}$ for the partial products of the weight matrices $W_m\cdots W_\ell$, so that $Z_m = W_{m:\ell+1} Z_\ell$ for all $m\geq\ell\in\{L_1+1,\dots,L_1+L_2\}$.


We index individual samples in a feature matrix $Z$ as $z_{ci},$ meaning the $i$-th sample of the $c$-th class, sometimes adding an upper-index to denote the layer or matrix to which the sample belongs. Let $\mu_c$ denote the mean of all samples from class $c$ and $\mu_G$ the global mean. Let $\Bar{Z}$ be the matrix of class-means stacked into columns. The NC1 metric on a feature matrix $Z$ is given by  $\frac{\text{tr}(\Sigma_W)}{\text{tr}(\Sigma_B)},$ where $\Sigma_W=\frac{1}{N}\sum_{c, i}(z_{ci}-\mu_c)(z_{ci}-\mu_c)^\top$ and $\Sigma_B=\frac{1}{K}\sum_{c=1}^{K} (\mu_c-\mu_G)(\mu_c-\mu_G)^\top.$ The NC2 metric on a feature matrix $Z$ is defined as $\kappa(\bar{Z})$, i.e., the conditioning number of the class-mean matrix of $Z.$ 
The NC3 metric on a feature matrix $Z$ and a weight matrix $W$ is defined as $\frac{1}{N}\sum_{c,i}\cos(z_{ci}, W_{c:}),$ i.e., the average cosine similarity of features and weight vectors corresponding to the features' class.
%

At this point, we 
state our 
result giving a set of sufficient conditions for NC1, NC2 and NC3.

\begin{restatable}{theorem}{metaresult}\label{th:NC1}
If the network satisfies
\begin{itemize}[leftmargin=6mm]
\item approximate interpolation, i.e., $\left\Vert Z_{L}-Y\right\Vert _{F}\leq\epsilon_{1}$,
\item approximate balancedness, i.e., $\left\Vert W_{\ell+1}^{\top}W_{\ell+1}-W_{\ell}W_{\ell}^{\top}\right\Vert _{op}\leq\epsilon_{2}$, for $\ell\in \{L_1+1, \ldots, L-1\}$,
\item bounded representations and weights, i.e., $\left\Vert Z_{L-2}\right\Vert _{op}, \left\Vert Z_{L-1}\right\Vert _{op}, \left\Vert W_{\ell}\right\Vert _{op}\leq r$, for $\ell\in \{L_1+1, \ldots, L\}$,
\end{itemize}
then if $\epsilon_1\le \min \left(s_K(Y),  \sqrt{\frac{(K-1)N}{4K}}\right)$, 
\begin{equation}\label{eq:NC1}
\text{NC1}(Z_{L-1})\le \frac{r^2}{N}\frac{(\Psi(\epsilon_1, \epsilon_2, r))^2}{\left(\sqrt{\frac{K-1}{K}}-\frac{2}{\sqrt{N}}\epsilon_1\right)^2} ,
\end{equation}
where $\Psi(\epsilon_1, \epsilon_2, r)=r\left(\frac{\epsilon_{1}}{s_K(Y)-\epsilon_1}+\sqrt{n_{L-1} \epsilon_2}\right)$.
If we additionally assume that the linear part of the network is not too ill-conditioned, i.e., $\kappa(W_{L:L_1+1})\le c_3$, then
\begin{equation}\label{eq:NC2}
    \kappa(W_L)\le c_3^{\frac{1}{L_2}}(1+\epsilon)^{\frac{1}{L_2}}+{c_3^{\frac{1}{L_{2}}-1}}\epsilon,
\end{equation}
with 
$\epsilon \hspace{-.1em}=\hspace{-.1em} \frac{\frac{L_2^2}{2}r^{2(L_2-1)}\epsilon_2}{\frac{(s_{K}(Y)-\epsilon_{1})^{2}}{\left\Vert X\right\Vert _{op}^{2}r^{2L_{1}}}-\frac{L_2^2}{2}r^{2(L_2-1)}\epsilon_2}$. 
Finally, 
\begin{align}\label{eq:NC2*}
    \text{NC2}(Z_{L-1})&\le 
    \frac{\kappa(W_L)+s_K(Y)^{-1}r\Psi(\epsilon_1, \epsilon_2, r)}{1-s_K(Y)^{-1}r\Psi(\epsilon_1, \epsilon_2, r)} \\ 
%
%
\label{eq:NC3}
\text{NC3}(Z_{L-1}, W_L)&\ge  \frac{(\sqrt{N}-\epsilon_1)^2+N(\kappa(W_L))^{-2}-\left(r\Psi(\epsilon_1, \epsilon_2, r)+\sqrt{K}(\kappa(W_L)^2-1)\right)^2}{2N\kappa(W_L)(1+\epsilon_1)}.
\end{align}
\end{restatable}

The restrictions $\epsilon_1\le s_K(Y)$ and 
$\epsilon_1 \le \sqrt{\frac{(K-1)N}{4K}}$ are mild, and we are interested in the regime in which $\epsilon_1$ is small. In words, \eqref{eq:NC1} shows that, when  $\epsilon_1, \epsilon_2\approx 0$ (i.e., the network approximately interpolates the data in a balanced way), $\Psi(\epsilon_1, \epsilon_2, r)\approx 0$ and the within-class variability (which captures NC1) vanishes.
If in addition the depth of the linear part of the network grows, the RHS of \eqref{eq:NC2} approaches 1, i.e., the last weight matrix $W_L$ is close to orthogonal. This implies that \emph{(i)} $Z_{L-1}$ is also close to orthogonal (which captures NC2), and \emph{(ii)}
the weights in the last layer align with $Z_{L-1}$ (which captures NC3). In fact, when 
$\Psi(\epsilon_1, \epsilon_2, r)\approx 0$ and $\kappa(W_L)\approx 1$, the RHS of both \eqref{eq:NC2*} and \eqref{eq:NC3} is close to $1$. 
Below we give a proof sketch deferring the complete argument to Appendix \ref{app:deferred}.


\paragraph{Proof sketch.} We start with NC1. If $Z_L$ is already well-collapsed (which is guaranteed by the approximate interpolation) and $W_L$ is well-conditioned, then the only source of within-class variability in $Z_{L-1}$ is within the null space of $W_L.$ However, if $W_L$ and $W_{L-1}$ are balanced, the image of $Z_{L-1}$ must be approximately in a subspace of the row space of $W_L$ and, hence, $Z_{L-1}$ has little freedom within the kernel of $W_L$. More formally, consider first the case of perfect balancedness (i.e., $\epsilon_2=0$), and denote by $W_L^+$ the pseudo-inverse of $W_L$. Then, $\text{Im}(Z_{L-1}) \subset \text{Im}(W_L^\top)$ and 
\begin{equation}\label{eq:normW}    
\norm{W_L^+W_LZ_{L-1}-W_L^{+}Y}_F=\norm{Z_{L-1}-W_L^{+}Y}_F\le \frac{\epsilon_1}{s_K(W_L)}.
\end{equation}
As $s_K(W_L)$ can be lower bounded by using the assumptions on approximate interpolation and boundedness of representations, the RHS of \eqref{eq:normW} is small and, therefore, $Z_{L-1}$ is close to a matrix with zero within-class variability. Moving to the case  $\epsilon_2\neq 0$, we need to show that $W_L^{+}W_LZ_{L-1}$ is close to $Z_{L-1}.$ As $W_L^{+}W_L$ projects onto the row-space of $W_L,$ only the part of $Z_{L-1}$ in the kernel of $W_L$ has to be considered. This part is controlled after writing $Z_{L-1}=W_{L-1}Z_{L-2}$, using the boundedness of $Z_{L-2}$ and the approximate balancedness between $W_L$ and $W_{L-1}.$ Finally, as $Z_{L-1}$ is close to $ W_L^+Y,$ a direct computation yields the bound on NC1. 

Next, to bound $\kappa(W_L),$ we notice that $(W_LW_L^\top )^{L_2}-(W_{L:L_1+1}W_{L:L_1+1}^\top)$ has small operator norm, since the weights of linear layers are approximately balanced. This allows to upper bound $\kappa(W_L)^{2L_2}$ in terms of $\kappa(W_{L:L_1+1})$ (plus a small perturbation), which gives \eqref{eq:NC2}. 

To lower bound the NC3 metric, we rescale $Z_{L-1}$ and $W_L$ to $Z'_{L-1}$ and $W'_L$, so that their columns and rows, respectively, have roughly equal size. Then, we reformulate the problem to proving that $\left\langle Z'_{L-1}, W'_LY\right\rangle$ is close to its theoretical maximum. We proceed to show this by arguing that, in this scaling and given that $W_L$ is sufficiently well-conditioned, $W'_L$ can be replaced by $(W'_L)^+$. As $Z_{L-1}$ is close to $ W_L^+Y$, we obtain~\eqref{eq:NC3}. Finally, the bound on NC2 in \eqref{eq:NC2*} follows by combining \eqref{eq:NC2} with the closeness between $Z_{L-1}$ and $W_L^+Y$ already obtained in the proof of NC1. 

\section{Gradient Descent Leads to No Within-Class Variability (NC1)} \label{sec:gd_guarantee}

In this section, we show that NC1 holds for a class of neural networks with one wide layer followed by a pyramidal topology, as considered in \cite{QuynhMarco2020}. To do so, we show that the balancedness and interpolation conditions of Theorem \ref{th:NC1} holds. We expect that similar conditions -- and, therefore, NC1 -- also hold under minimal over-parameterization and for ReLU networks, adapting e.g.\ the approach of \cite{bombari2022memorization} and \cite{ZouGu2019}, respectively. 

We consider a neural network as in \eqref{eq:F_l}, and we minimize the $\lambda$-regularized square loss
$    C_\lambda(\theta) = \frac{1}{2}\norm{z_{L}(\theta)-y}_2^2+\frac{\lambda}{2}\norm{\theta}_2^2$, where $z_{L}$ and $y$ are obtained by vectorizing $Z_{L}$ and $Y$, respectively, and $\theta$ collects all the parameters of the network. To do so, we consider the gradient descent (GD) update $\theta_{k+1}=\theta_k - \eta\nabla C_\lambda(\theta_k)$, where $\eta$ is the step size and $\theta_k=(W_\ell^k)_{\ell=1}^L$ contains all parameters at step $k.$ We also denote by $Z_\ell^k$ the output of layer $\ell$ after $k$ steps of GD. 
We make the following assumption on the pyramidal topology of the network, noting that this requirement is also common in prior work on the loss landscape \citep{QuynhICML2017,QuynhICML2018}. 

\begin{assumption}(Pyramidal network topology)\label{ass:net_topo}
    Let $n_1\geq N$ and $n_{2}\geq n_{3}\geq\ldots\geq n_{L}.$
\end{assumption}

We make the following assumptions on the activation function $\sigma$ of the non-linear
layers.

\begin{assumption}(Activation function)\label{ass:act} 
    Fix $\gamma\in (0, 1)$ and $\beta \ge 1$. 
    Let $\sigma$ satisfy that:
    (i) $\sigma'(x)\in[\gamma,1]$, (ii) $|\sigma(x)|\leq|x|$ for every $x\in\RR$, and (iii) $\sigma'$ is $\beta$-Lipschitz. 
\end{assumption}
In principle, $\sigma$ can change at all layers, as long as it satisfies the above assumption. As an example, one can consider a family of parameterized ReLU functions, smoothened by a Gaussian kernel:
\begin{align}\label{eq:smooth_lrelu}
    \sigma(x) = -\frac{(1-\gamma)^2}{2\pi\beta}\hspace{-.1em} + \hspace{-.1em}\frac{\beta}{1-\gamma}\hspace{-.1em}\int_{-\infty}^{\infty}\hspace{-1.25em} \max(\gamma u, u)\, e^{-\frac{\pi\beta^2(x-u)^2}{(1-\gamma)^2}} du.
\end{align}
One can readily verify that the activation in \eqref{eq:smooth_lrelu} satisfies Assumption \ref{ass:act} and it uniformly approximates the ReLU function over $\mathbb R$, see Lemma B.1 in \cite{QuynhMarco2020}.
Next, let us introduce some notation\footnote{To avoid confusion, we note that this notation is different from the one used in \citep{QuynhMarco2020}.} for the singular values of the weight matrices at initialization $\theta_0=(W_\ell^0)_{\ell=1}^{L}$:
\begin{equation}\label{eq:notation_init}
\begin{split}
&\lambda_\ell=\svmin{W_\ell^0},\,\,\,\,\bar{\lambda}_\ell=\norm{W_\ell^0}_{op}+\min_{\ell\in \{3, \ldots, L\}}\lambda_\ell,\,\,\,\, 
 \lambda_{i\to j} = \prod_{\ell=i}^j \lambda_\ell, \,\,\,\,\bar{\lambda}_{i\to j} = \prod_{\ell=i}^j\bar{\lambda}_\ell.\\
\end{split}
\end{equation}
We also define $\lambda_F=\svmin{\sigma(W_1^0X)}$ as the smallest singular value of the output of the first hidden layer at initialization. 
Finally, we make the following assumption on the initialization. 
\begin{assumption}(Initial conditions)\label{ass:init} 
\begin{align}\label{eq:assinit}
\lambda_F \lambda_{3\to L}\min(\lambda_F, \min_{\ell\in \{3, \ldots, L\}} \lambda_\ell)\ge 8\gamma\sqrt{\left(\frac{2}{\gamma}\right)^L C_0(\theta_0)}.
\end{align}
\end{assumption}
We note that \eqref{eq:assinit} can be satisfied by choosing a sufficiently small initialization for the second layer and a sufficiently large one for the remaining layers. In fact, the LHS of \eqref{eq:assinit} depends on all the layer weights except the second, so this quantity can be made arbitrarily large. Next, by taking a sufficiently small second layer, the term $\sqrt{2C_0(\theta)}=\|Z_L-Y\|_F$ can be upper bounded by $2\|Y\|_F$. As a consequence, the RHS of \eqref{eq:assinit} is at most $8\sqrt{2}\|Y\|_F\gamma\big(\frac{2}{\gamma}\big)^{L/2}$. As the LHS of \eqref{eq:assinit} can be arbitrarily large, the inequality holds for a suitable initialization. 

 
\begin{restatable}{theorem}{mainGD}\label{thm:main}
    Let the network satisfy Assumption \ref{ass:net_topo}, 
    $\sigma$ satisfy Assumption \ref{ass:act} 
    and the initial conditions satisfy Assumption \ref{ass:init}. 
Fix $0<\epsilon_1\le \frac{1}{2}\sqrt{\frac{(K-1)N}{K}}$, $\epsilon_2>0$, let $b\ge 1$ be s.t.\ $\|X_{:i}\|_2\le b$ for all $i$, and run $k$ steps of $\lambda$-regularized GD with step size $\eta$, where
\begin{equation}\label{eq:ublambdaeta}
    \begin{split}
\lambda &\leq \min \left( 2\left(\frac{\gamma}{2}\right)^{L-2}\lambda_F\lambda_{3\to L}, \frac{2C_0(\theta_0)}{\norm{\theta_0}_2^2}, \frac{\epsilon_1^2}{18(\norm{\theta_0}_2+\lambda_F/2)^2} \right),\\
\eta&\leq\min\left( \frac{1}{2\beta_1},\frac{1}{5N\beta b^{3}\max\left(1, \left(\frac{2\epsilon_1^2}{\lambda}\right)^{3L/2}\right)L^{5/2}},\frac{1}{2\lambda},\left(\frac{\lambda}{2\epsilon_1^2}\right)^{L_1+L}\frac{\epsilon_2}{4\|X\|_{op}^2}\right),\\
k&\ge \left\lceil \frac{\log\frac{\lambda m_{\lambda}}{C_{\lambda}(\theta_{0})-\lambda m_{\lambda}}}{\log(1-\eta\frac{\alpha}{8})}\right\rceil+\left\lceil \frac{\log\frac{\lambda\epsilon_2}{4\epsilon_1^2}}{\log(1-\eta\lambda)}\right\rceil,
    \end{split}
\end{equation}
with $\beta_1 = 5N\beta b^{3}\left(\prod_{\ell=1}^L\max(1, \bar\lambda_\ell)\right)^3 L^{5/2}$, $m_{\lambda}=(1+\sqrt{4\lambda/\alpha})^{2}\left(\left\Vert \theta_{0}\right\Vert_2 +r_0\right)^{2}$, $r_0=\frac{1}{2}\min(\lambda_F, \min_{\ell\in \{3, \ldots, L\}} \lambda_\ell)$, and $\alpha=2^{-(L-3)}\gamma^{L-2}\lambda_F\lambda_{3\to L}$. 
Then, we have that 
\begin{equation}\label{eq:NC1GD}
\text{NC1}(Z_{L-1}^k)\le \frac{r^2}{N}\frac{\Psi(\epsilon_1\sqrt{2}, \epsilon_2, r)}{\left(\sqrt{\frac{K-1}{K}}-\frac{2\sqrt{2}}{\sqrt{N}}\epsilon_1\right)^2},
\end{equation}
with $\Psi(\epsilon_1, \epsilon_2, r)=r\left(\frac{\epsilon_{1}}{s_K(Y)-\epsilon_1}+\sqrt{n_{L-1} \epsilon_2}\right)$ and
\begin{equation}\label{eq:defr}
r=\max\left(\epsilon_1\sqrt{\frac{2}{\lambda}}, \left(\epsilon_1\sqrt{\frac{2}{\lambda}}\right)^{L-2}\|X\|_{op}, \left(\epsilon_1\sqrt{\frac{2}{\lambda}}\right)^{L-1}\|X\|_{op}\right).
\end{equation} 
\end{restatable}

In words, if regularization and learning rate are small enough and we run GD for sufficiently long (as in \eqref{eq:ublambdaeta}), then the within-class variability vanishes (as in \eqref{eq:NC1GD}). To interpret the result, note that the NC1 metric tends to $0$ (i.e., the collapse is perfect) as $\epsilon_1, \epsilon_2\to 0$. Since the terms $\beta_1, m_\lambda, r_0, \alpha$ do not depend on $\epsilon_1, \epsilon_2$ (but just on the network architecture and initialization), taking $\lambda$ of order $\epsilon_1^2$ and $\eta$ of order $\epsilon_2$ satisfies the first two requirements in \eqref{eq:ublambdaeta}, also giving that $r$ in \eqref{eq:defr} is of constant order. Finally, as $\log(1-x)\approx -x$ for small $x$, the quantity $\eta k$ -- which quantifies the time of the dynamics, since it is the product of learning rate and number of GD steps -- is of order $\log(1/\lambda)+\log(1/\epsilon_2)/\lambda$. Below we provide a proof sketch deferring the full argument to Appendix \ref{app:deferred}.

\paragraph{Proof sketch.} We show that the network trained via $\lambda$-regularized GD fulfills the three sufficient conditions for NC1 given by Theorem \ref{th:NC1}, i.e., approximate interpolation, approximate balancedness and bounded representations/weights. To do so, we distinguish two phases in the training dynamics.

The \emph{first phase} lasts for logarithmic time in $1/\lambda$ (or, equivalently, $1/\epsilon_1$) and, here, the loss decreases exponentially fast to a value of at most $2\lambda m_\lambda\le \epsilon_1^2$. As the learning rate is small enough, the loss cannot increase during the GD dynamics, which already gives approximate interpolation. To show the exponential convergence, we proceed in two steps. First, Lemma 4.1 in \citep{QuynhMarco2020} gives that the unregularized loss $C_0(\theta)$ satisfies the Polyak-Lojasiewicz (PL) inequality
\begin{equation}\label{eq:PL}   
\left\Vert \nabla C_0(\theta)\right\Vert_2 ^{2}\geq\frac{\alpha}{2}C_0(\theta),
\end{equation}
for all $\theta$ in a ball centered at initialization $\theta_0$ and with sufficiently large radius (captured by $r_0$). Next, we show that, if $C_0(\theta)$ satisfies the $\alpha$-PL inequality in \eqref{eq:PL}, then the regularized loss $C_{\lambda}(\theta)=C_0(\theta)+\frac{\lambda}{2}\left\Vert \theta\right\Vert_2^{2}$
satisfies a 
shifted $\alpha$-PL inequality, which implies exponential convergence. This second step is formalized by the proposition below proved in Appendix \ref{app:deferred}.

\begin{restatable}{prop}{PL}\label{prop:PL}
Let $C_0(\theta)$ satisfy the $\alpha$-PL inequality \eqref{eq:PL} in the ball $B(\theta_{0},r_0)$.
Then, in the same ball, $C_{\lambda}(\theta)$ satisfies the inequality
\begin{equation}\label{eq:shiftedPL}    
\left\Vert \nabla C_{\lambda}(\theta))\right\Vert_2 ^{2}\geq\frac{\alpha}{4}\left(C_{\lambda}(\theta)-\lambda m_{\lambda}\right),
\end{equation}
where $m_{\lambda}=(1+\sqrt{4\lambda/\alpha})^{2}\left(\left\Vert \theta_{0}\right\Vert_2 +r_0\right)^{2}$.
Furthermore, assume that $r_0\geq 8\sqrt{C_{\lambda}(\theta_{0})/\alpha}$ and $\nabla C_0(\theta)$ is $\beta_1$-Lipschitz in $B(\theta_{0}, r_0)$. Then, for any $\eta<1/(2\beta_1)$, there exists
\begin{equation}    \label{eq:Tbound}
k_1\leq\left\lceil \frac{\log\frac{\lambda m_{\lambda}}{C_{\lambda}(\theta_{0})-\lambda m_{\lambda}}}{\log(1-\eta\frac{\alpha}{8})}\right\rceil
\end{equation}
such that the $k_1$-th iterate of GD satisfies
\begin{equation}\label{eq:GDloss}
 C_{\lambda}(\theta_{k_1})\leq2\lambda m_{\lambda},\qquad\qquad  \left\Vert \theta_{k_1}-\theta_{0}\right\Vert_2 \leq 8\sqrt{\frac{C_{\lambda}(\theta_{0})}{\alpha}}\le r_0. 
\end{equation}
\end{restatable}

The \emph{second phase} lasts for linear time in $1/\lambda$ (or, equivalently in $1/\epsilon_1^2$) and logarithmic time in $1/\epsilon_2$ and, here, the weight matrices in the linear part of the network become balanced. More precisely, we show that, if $W_{\ell}$ is a weight matrix of the linear part, $\|W_{\ell+1}^\top W_{\ell+1}
- W_{\ell} W_{\ell}^\top\|_{op}$ decreases exponentially and the exponent scales with $1/\lambda$, which gives approximate balancedness.
Finally, as the regularization term in the loss is at most $\epsilon_1^2$, the operator norm of representations and weight matrices is bounded by $r$ as in \eqref{eq:defr}, and the proof is completed by an application of Theorem \ref{th:NC1}. 



\section{Orthogonality of Class Means (NC2) and Alignment with Last Weight Matrix (NC3)}

To guarantee the orthogonality of class means and their alignment with the last weight matrix, the crux is to show that the condition number $\kappa(W_L)$ of the last weight matrix $W_L$ is close to one. In fact, as $Z_{L-1} \approx W_L^+ Y$, this implies that the last hidden representation $Z_{L-1}$ is approximately a rotation and rescaling of $Y$, which gives NC2, and a bound on NC3 of the form in \eqref{eq:NC3} also follows. 

The fact that $\kappa(W_L)\approx 1$ is a consequence of the presence of many balanced linear layers. 
Indeed, balancedness implies
$W_L W_L^\top = (W_{L:L_1+1}W_{L:L_1+1}^\top)^\frac{1}{L_2}$,
which gives that 
$\kappa(W_L) = \kappa(W_{L:L_1+1})^\frac{1}{L_2}$.
Thus, if the conditioning of the product of the linear layers $W_{L:L_1+1}$ can be bounded independently of $L_2$, one can guarantee that the conditioning of $W_L$ approaches $1$ as $L_2\to\infty$.
As this is difficult to obtain in full generality (in particular the assumptions of Theorem \ref{th:NC1} may not be sufficient), we show that the conditioning can be controlled \emph{(i)} at any global minimizer, and \emph{(ii)} when the parameters are `stable' under large learning rates. 

\subsection{Global Minimizers}
We first show that any set of parameters that approximately interpolate with small norm has bounded condition number. We will then show that with the right choice of widths and ridge, all global minimizers satisfy these two assumptions. 

\begin{restatable}{prop}{goodlossimpliesnctwo}
\label{prop:conditioning_from_interp+param_norm}Let $\sigma$ satisfy Assumption \ref{ass:act}. Then, for any network that satisfies
\begin{itemize}[leftmargin=6mm]
\item approximate interpolation, i.e., $\left\Vert Z_L-Y\right\Vert _{F}\leq\epsilon_{1}$,
\item bounded parameters, i.e., $\left\Vert \theta\right\Vert_2^{2}\leq LK+c$,
\end{itemize}
the linear part $W_{L:L_{1}+1}$ satisfies
\begin{equation}    
\kappa(W_{L:L_{1}+1})\leq\exp\left(\frac{1}{2}\left(c+L_{1}K\log K-2K\log\frac{s_K(Y)-\epsilon_{1}}{\left\Vert X\right\Vert _{op}}\right)\right).
\end{equation}
\end{restatable}

\begin{restatable}{theorem}{globalisnc}\label{thm:global_is_nc}
Let $\sigma$ satisfy Assumption \ref{ass:act}. Assume there exist parameters of the nonlinear part $\theta_{nonlin}=(W_\ell)_{\ell=1}^{L_1}$ such that $Z_{L_{1}}=Y$ and $\|\theta_{nonlin}\|_2^2=c$. Then, at any global minimizer of the regularized loss $\mathcal{L}_{\lambda}(\theta)=\frac{1}{2}\left\Vert Y-Z_{L}\right\Vert _{F}^{2}+\frac{\lambda}{2}\left\Vert \theta\right\Vert_2 ^{2}$
with $\lambda\leq\frac{\epsilon_{1}^{2}}{KL+c}$, we have
\begin{equation}\label{eq:ubkappa}
\begin{split}    
\kappa(W_{L:L_{1}+1})&\leq\left(\frac{\left\Vert X\right\Vert _{op}}{s_K(Y)-\epsilon_{1}}\right)^{K}\exp\left(\frac{1}{2}\left(c-L_{1}K+L_{1}K\log K\right)\right),\\
\kappa(W_{L})&\leq\left(\frac{\left\Vert X\right\Vert _{op}}{s_K(Y)-\epsilon_{1}}\right)^{\frac{K}{L_1}}\exp\left(\frac{1}{2L_1}\left(c-L_{1}K+L_{1}K\log K\right)\right).
\end{split}
\end{equation}
This implies that the bounds on NC1, NC2 and NC3 in \eqref{eq:NC1}, \eqref{eq:NC2*} and \eqref{eq:NC3}, respectively, hold with $\kappa(W_L)$ upper bounded as above and $\epsilon_2=0$.
\end{restatable}


The assumption that the parameters of the nonlinear part can be chosen to fit the labels $Y$ is guaranteed for large enough width \emph{(i)} by relying on any traditional approximation result \citep{Hornik1989,Leshno,arora_2018_relu_piecewise_lin,he_2018_relu_piecewise_lin}, or \emph{(ii)} by taking the infinite time limit of any convergence results \citep{QuynhMarco2020}, or \emph{(iii)} by taking the limit $\lambda\searrow 0$ in Proposition \ref{prop:PL}. A sketch of the arguments is below, with full proofs deferred to Appendix \ref{app:deferred}.

\paragraph{Proof sketch.} To prove Proposition \ref{prop:conditioning_from_interp+param_norm}, we write the norm of the linear and nonlinear parts of the network in terms of the conditioning of $W_{L:L_1+1}$: for the linear part, this is a direct computation; for the nonlinear part, we use Theorem 1 of \cite{dai_2021_repres_cost_DLN}, which lower bounds the norm of the parameters in terms of the product of the singular values, and then manipulate the latter quantity to obtain again the desired conditioning. Next, to prove Theorem \ref{thm:global_is_nc}, we pick the parameters of the nonlinear part $\theta_{nonlin}$ s.t.\ $Z_{L_{1}}=Y$ and $\|\theta_{nonlin}\|_2^2=c$, and set the linear layers to the identity. This leads to a total parameter norm of $KL_2+c$ and a regularized cost of $\frac{\lambda}{2}(KL_2+c)$, and it forces the global minimizer to satisfy the assumptions of Proposition \ref{prop:conditioning_from_interp+param_norm}, which gives the claim on $\kappa(W_{L:L_1+1})$. Then, as all local minimizers have balanced linear layers, $\kappa(W_{L})=\kappa(W_{L:L_1+1})^\frac{1}{L_1}$, which gives \eqref{eq:ubkappa}. Finally, the claim on NC1, NC2 and NC3 follows from an application of Theorem \ref{th:NC1}.

\subsection{Large Learning Rates}
Previous works have observed that the learning rates used in practice
are typically `too large', i.e.\ the loss may not always
be strictly decreasing and 
GD 
diverges from
gradient flow \citep{cohen_2021_edge_of_stability}. Thankfully, instead of simply diverging, for large $\eta$ (but
not too large) the parameters naturally end up at the `edge of stability': 
the top eigenvalue of the Hessian $\mathcal{H}C_{\lambda}$
is close to $\frac{2}{\eta}$, i.e., the threshold below which
GD is stable \citep{cohen_2021_edge_of_stability,lewkowycz_2020_large_lr}. One can thus interpret GD with learning rate
$\eta$ as minimizing the cost $C_{\lambda}$ amongst parameters $\theta$
such that $\left\Vert \mathcal{H}C_{\lambda}\right\Vert _{op}<\frac{2}{\eta}$. These observations are supported by strong empirical evidence and have been also proved theoretically for
simple models \citep{damian2022_edge_stability_simple_model}, although a general result remains difficult to prove due to the chaotic behavior of GD for
large $\eta$. 
Specifically, the Hessian has the form
\begin{equation}
    \mathcal{H}C_{\lambda}(\theta)=\left(\nabla_{\theta}Z_{L}\right)^{\top}\partial_{\theta}Z_{L}+\mathrm{Tr}\left[(Y-Z_{L})\nabla_{\theta}^{2}Z_{L}\right]+\lambda I_P,
\end{equation}
with $\nabla_{\theta}Z_{L}\in \mathbb R^{P\times NK}$ and $P$ the number of parameters.
The first term is the Fisher information matrix: this is dual to
the Neural Tangent Kernel (NTK) $\Theta=\nabla_{\theta}Z_{L}\left(\nabla_{\theta}Z_{L}\right)^{\top}\in\mathbb R^{NK\times NK}$ \citep{jacot2018neural}. 
Therefore, at approximately
interpolating points, we have $\left\Vert \mathcal{H}C_{\lambda}(\theta)\right\Vert _{op}=\left\Vert \Theta\right\Vert _{op}+O(\epsilon_{1})+O(\lambda)$, where $\epsilon_1$ is the interpolation error and $\lambda$ the regularization parameter.
We can thus interpret large learning as forcing a bound on the operator norm of the NTK. For networks that approximately interpolate the data with bounded NTK and bounded weights, the following proposition guarantees good conditioning of the weights in the linear part (and therefore NC2-3).

\begin{restatable}{prop}{largelearningrates}\label{thm:large_learning_rates}
For any network that satisfies
\begin{itemize}[leftmargin=6mm]
\item bounded NTK, i.e., $\left\Vert \Theta\right\Vert _{op}=\max_{A}\frac{\left\Vert \nabla_{\theta}\mathrm{Tr}\left[Z_{L}A^{T}\right]\right\Vert_2 ^{2}}{\left\Vert A\right\Vert _{F}^{2}}\leq CL_{2}$,
\item approximate interpolation, i.e., $\left\Vert Z_{L}-Y\right\Vert _{F}\leq\epsilon_{1}$,
\item bounded weights, i.e., $\left\Vert W_{\ell}\right\Vert _{op}\leq r$,
\end{itemize}
for any $M\leq L_{2}$, there is $\ell\in\{L_{1}+1,\dots,L_{1}+M\}$
such that $\kappa(W_{L:\ell})\leq\frac{\sqrt{CL_{2}}Kr}{\sqrt{M}\left(s_{K}(Y)-\epsilon_{1}\right)}$.

Furthermore, any network that satisfies approximate interpolation and bounded weights is such that
\begin{equation}\label{eq:NTKlb}
 \left\Vert \Theta\right\Vert _{op}\geq\frac{\left(s_{K}(Y)-\epsilon_{1}\right)^{2}}{K^{2}r^{2}}L_{2} .  
\end{equation}
\end{restatable}

As an example, by choosing $M=\frac{L_2}{2}$, we guarantee that there is at least one layer $\ell$ in the first half of the linear layers s.t.\  $\kappa(W_{L:\ell})\leq\frac{\sqrt{2C}Kr}{s_{K}(Y)-\epsilon_{1}}$. Now, assuming the `edge of stability' phenomenon, a learning rate of $\eta=\frac{\eta_{0}}{L_{2}}$ implies a bound
$\left\Vert \Theta\right\Vert _{op} \leq \left\Vert \mathcal{H} C_\lambda \right\Vert _{op}  + O(\epsilon_1) + O(\lambda) \leq \frac{2 L_{2}}{\eta_0} + O(\epsilon_1) + O(\lambda)$, and thus Proposition \ref{thm:large_learning_rates} implies that there is a linear layer with bounded conditioning. Proposition \ref{thm:large_learning_rates} also suggests that one cannot take any significantly larger learning rate, since for any parameters that $\epsilon_{1}$-interpolate
and have $r$-bounded weights, the NTK satisfies the lower bound in \eqref{eq:NTKlb}, which implies that 
the learning rate must 
be smaller than $\frac{2r^{2}K^{2}}{L_{2}\left(s_{K}(Y)-\epsilon_{1}\right)^{2}}$. 

\paragraph{Proof sketch.}
The idea 
is that, by suitably choosing $A$ in the evaluation of $\frac{\left\Vert \nabla_{\theta}\mathrm{Tr}\left[Z_{L}A^{T}\right]\right\Vert_2 ^{2}}{\left\Vert A\right\Vert _{F}^{2}}$,  the operator norm of the NTK is lower bounded by some constant times $\sum_{\ell=L_{1}+1}^{L}\kappa(W_{L:\ell})^2$. This implies the desired upper bound on $\kappa(W_{L:\ell})$ for some $\ell\in \{L_1+1, \ldots, L_1+M\}$, as well as the lower bound on $\|\Theta\|_{op}$ in \eqref{eq:NTKlb}. The details are deferred to Appendix \ref{app:deferred}.

\section{Numerical Results}\label{sec:experiments}


In all experiments, we consider MSE loss and standard weight decay regularization. We train an MLP and a ResNet20 with an added MLP head on standard datasets (MNIST, CIFAR10), considering as backbone the first two layers for the MLP and the whole architecture before the linear head for the ResNet. 
We evaluate the following metrics related to neural collapse: for NC1, we compute $\text{tr}(\Sigma_W)/\text{tr}(\Sigma_B),$ where $\Sigma_W, \Sigma_B$ are the within- and between-class variability matrices of the feature matrices, respectively; for NC2, we display the conditioning number of the class-mean matrix; for NC3, we use the average cosine angle between the rows of a weight matrix and the columns of the preceding class-mean matrix. Finally, for balancedness we use $\frac{\left\Vert W_{\ell+1}^\top W_{\ell+1}-W_\ell W_\ell^\top\right\Vert_{op}}{\min\left\{\left\Vert W_{\ell+1}^\top W_{\ell+1}\right\Vert_{op}, \left\Vert W_\ell W_\ell^\top\right\Vert_{op}\right\}}$ and for negativity we use $\frac{\left\Vert Z_\ell-\sigma(Z_\ell)\right\Vert_{op}}{\left\Vert Z_\ell\right\Vert_{op}}.$  We measure such metrics (and also index the layers) starting from the output of the backbone. Our findings can be summarized as follows (see Appendix~\ref{app:experiments} for additional complementary experiments).

\paragraph{The deeper the linear head, the more clear NC occurs.}
We first test whether the models with deep linear heads exhibit NC and if that's the case, whether it gets better as we deepen the linear head. In Figure~\ref{fig:single_hyperparam_demo}, we show the NC metrics and the gram matrices of the class-mean matrices of the last layers for the training on CIFAR10 of ResNet20 with 6 extra layers of which the first three have a ReLU activation. We use weight decay of $0.001$ and learning rate of $0.001$, training for 5000 epochs (the learning rate drops ten-fold after 80\% of the epochs in all our experiments). 
\begin{figure}
    \centering
    \includegraphics[width=0.24\linewidth]{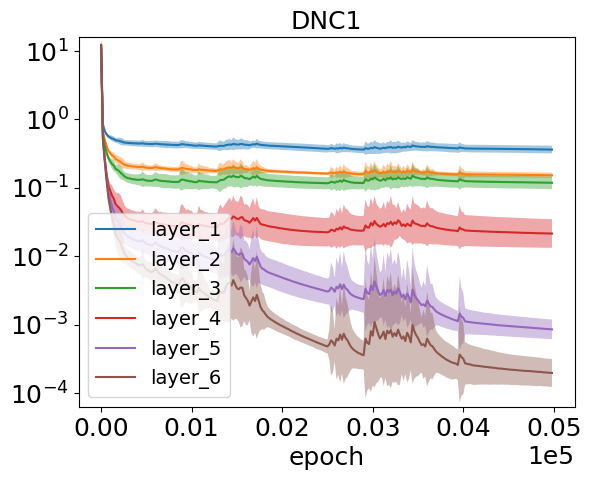}
    \includegraphics[width=0.24\linewidth]{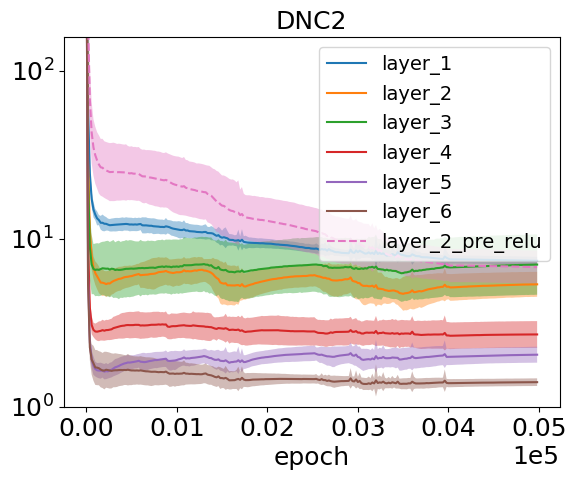}
    \includegraphics[width=0.24\linewidth]{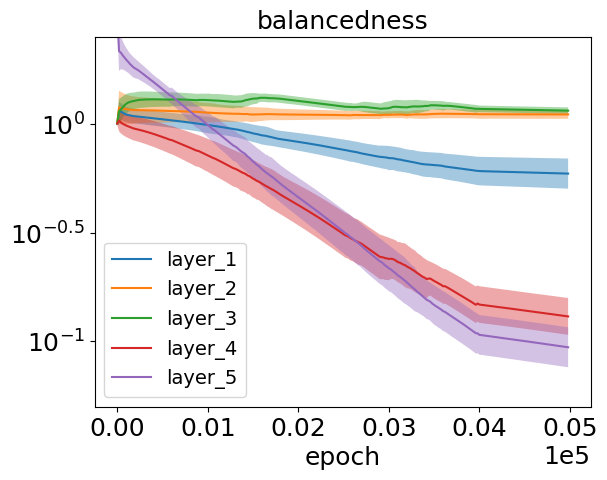}
    \includegraphics[width=0.24\linewidth]{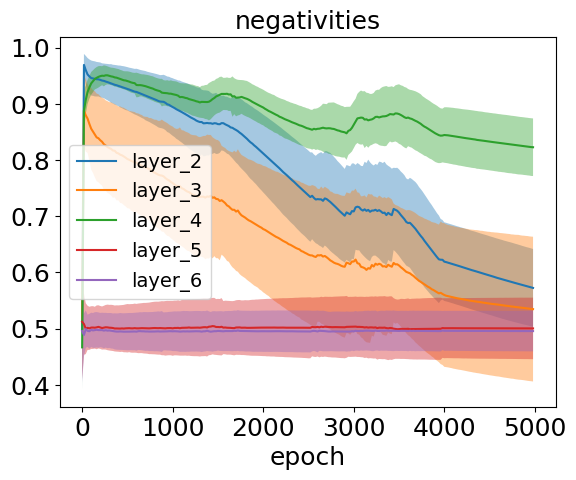}
    \includegraphics[width=0.24\linewidth]{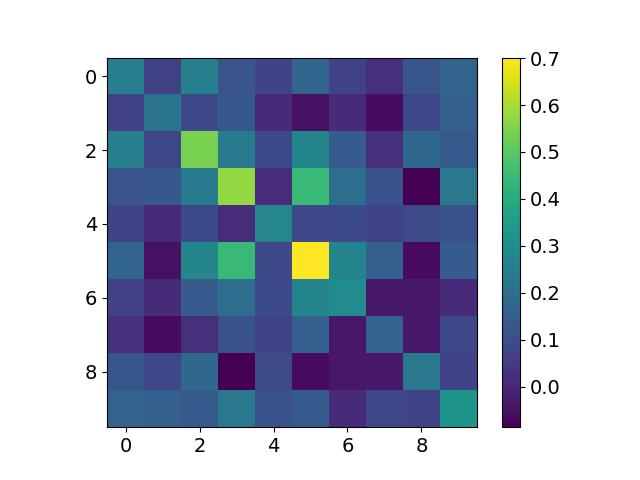}
    \includegraphics[width=0.24\linewidth]{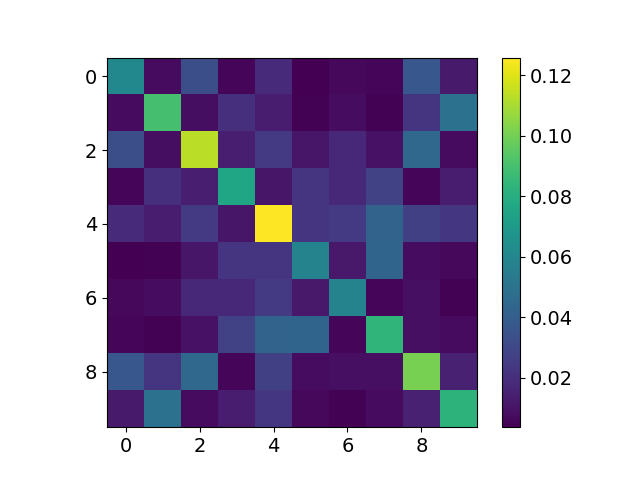}
    \includegraphics[width=0.24\linewidth]{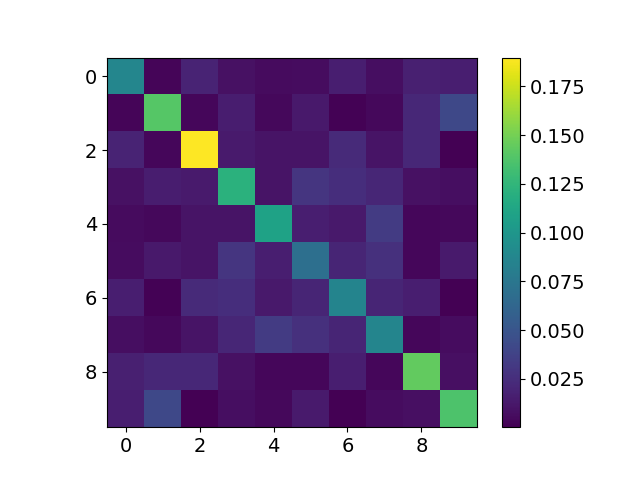}
    \includegraphics[width=0.24\linewidth]{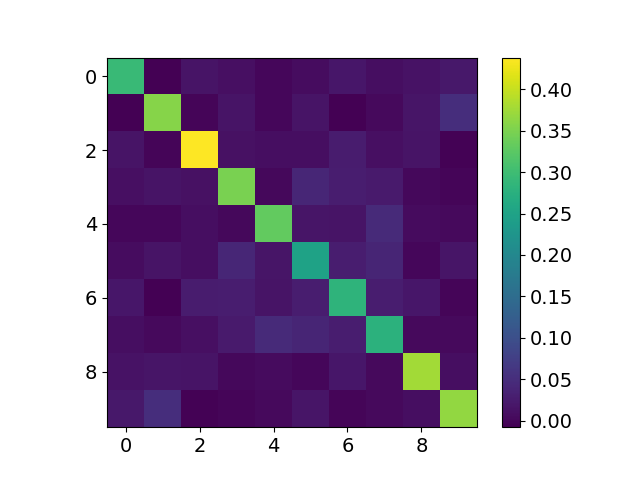}
    \caption{Last 7 layers of a 9-layer MLP trained on MNIST with weight decay $0.0018$ and learning rate $0.001$. \textbf{Top:} NC1s, NC2s, balancednesses and negativities, from left to right. Results are averaged over 5 runs, and the confidence band at 1 standard deviation is displayed. \textbf{Bottom:} Class-mean matrices of the last three layers (i.e., the linear head), the first before the last ReLU.}
    \label{fig:single_hyperparam_demo}
\end{figure}
The plot clearly shows that the collapse is reached throughout the training. We also see that the NC2 metric improves progressively with each layer of the linear head, as predicted by our theory. This effect is also clearly visible from the gram matrices of the class-means (bottom row of Figure~\ref{fig:single_hyperparam_demo}), which rapidly converge towards the identity. We note that these findings are remarkably consistent across a wide variety of hyperparameter settings and architectures. 

\begin{figure}
    \centering
    \includegraphics[width=0.24\linewidth]{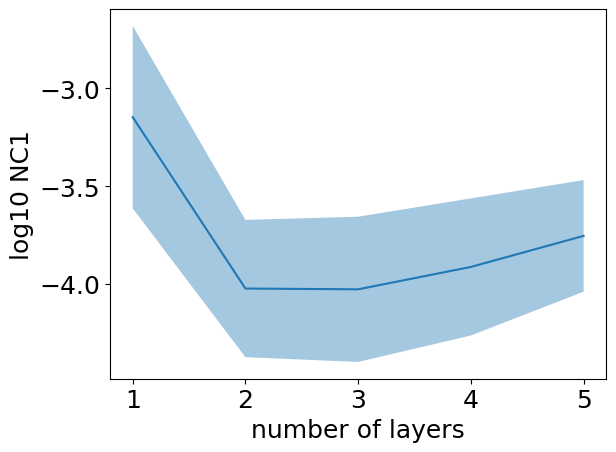}
    \includegraphics[width=0.24\linewidth]{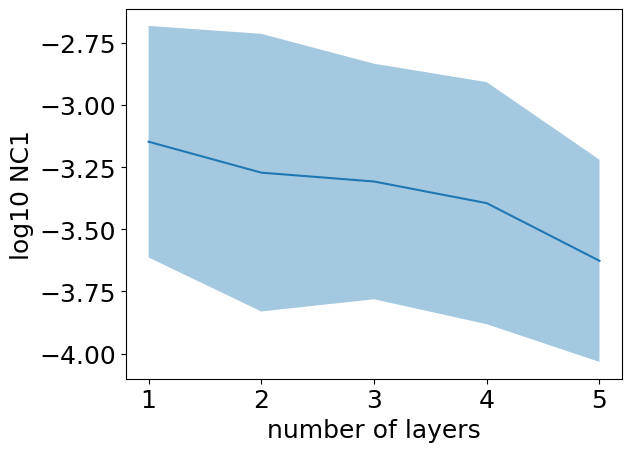}
    \includegraphics[width=0.24\linewidth]{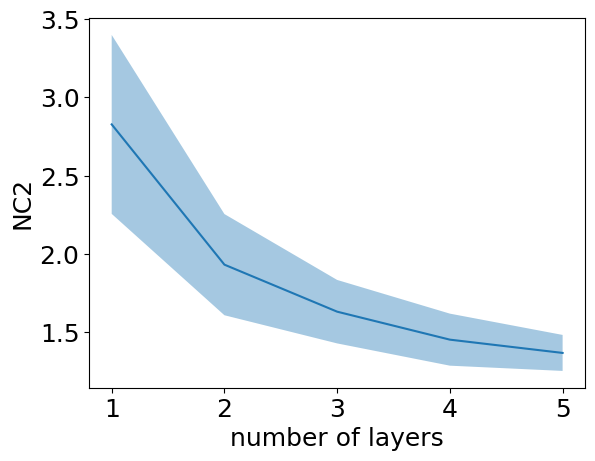}
    \includegraphics[width=0.24\linewidth]{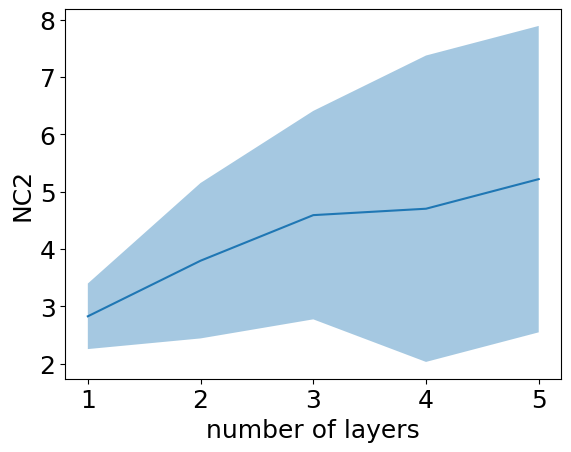}
    \includegraphics[width=0.24\linewidth]{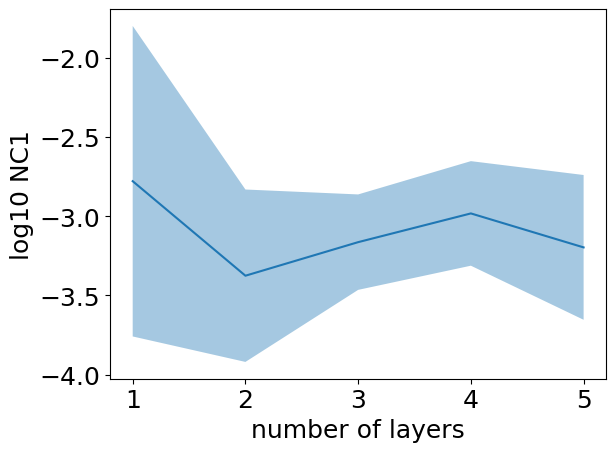}
    \includegraphics[width=0.24\linewidth]{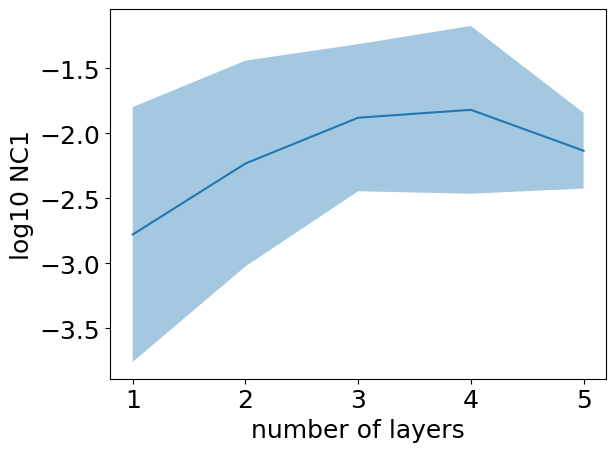}
    \includegraphics[width=0.24\linewidth]{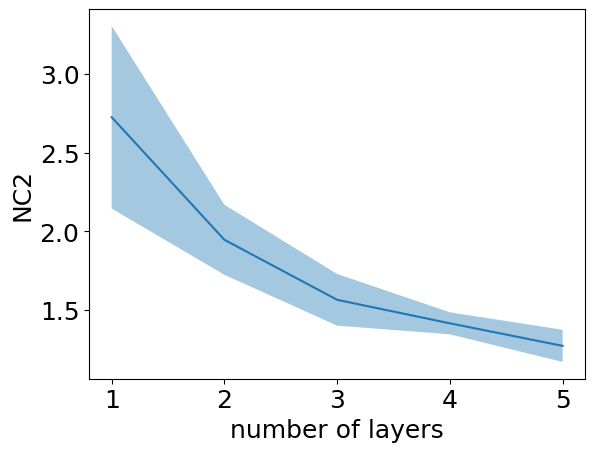}
    \includegraphics[width=0.24\linewidth]{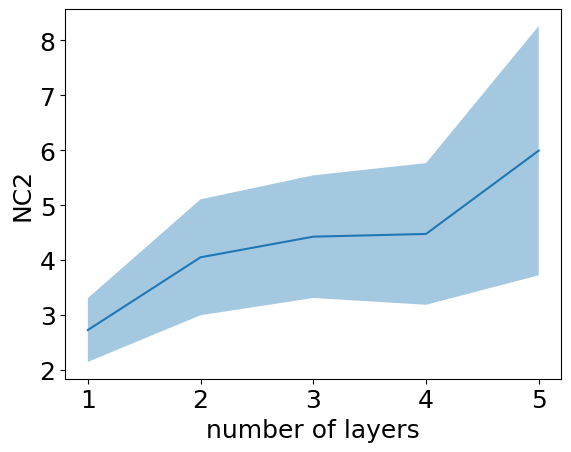}
\caption{\textbf{Upper/Lower row:} MLP/ResNet20 with a deep linear head. \textbf{Left to right:} NC1 in the last layer; NC1 in the first layer of the linear head; NC2 in the last layer; NC2 in the first layer of the linear head. All plots are a function of the number of layers in the linear head. Results are averaged over 50 runs (5 runs for each of the 10 hyperparameter setups), and the confidence band at 1 standard deviation is displayed.}
    \label{fig:nc_vs_depth}
\end{figure}

In Figure~\ref{fig:nc_vs_depth}, we plot the dependence of the NC metrics on the number of layers in the linear head. We train an MLP on MNIST with 5 non-linear layers and a  number of linear layers ranging from 1 to 5. We average over 5 runs per each combination of weight decay ($0.001, 0.004$) and with learning rate of $0.001$. We also train the ResNet20 on CIFAR10 with one non-linear layer head and 1 to 6 linear layers on top. We use the same weight decay and learning rate. The relatively high variance of the results is due to averaging over rather strongly different weight decay values used per each depth. The plots clearly show that the NC2 significantly improves in the last layer as the depth of the linear head increases, while it gets slightly worse in the input layer to the linear head. This is consistent with our theory. The NC1 does not have a strong dependence with the number of layers.

\paragraph{Linear layers are increasingly balanced throughout training.} Figure~\ref{fig:single_hyperparam_demo} shows that the metric capturing the balancedness exponentially decreases until it plateaus at a rather small value (due to the large learning rate) and then again it exponentially decreases at a smaller rate after reducing the learning rate. 
The balancedness of non-linear layers instead plateaus at a significantly larger value. 

\paragraph{Non-linear layers are increasingly balanced and linear, as the depth of the non-linear part increases.} 
\begin{figure}
    \centering
    \includegraphics[width=0.24\linewidth]{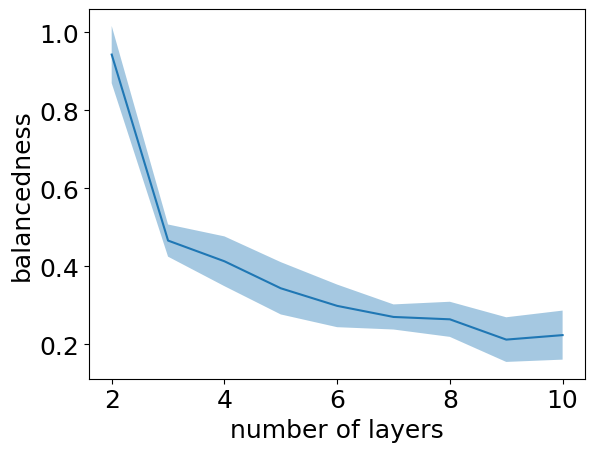}
    \includegraphics[width=0.24\linewidth]{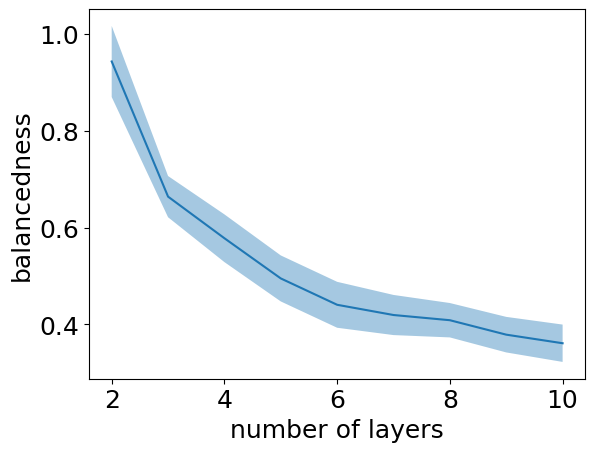}
    \includegraphics[width=0.24\linewidth]{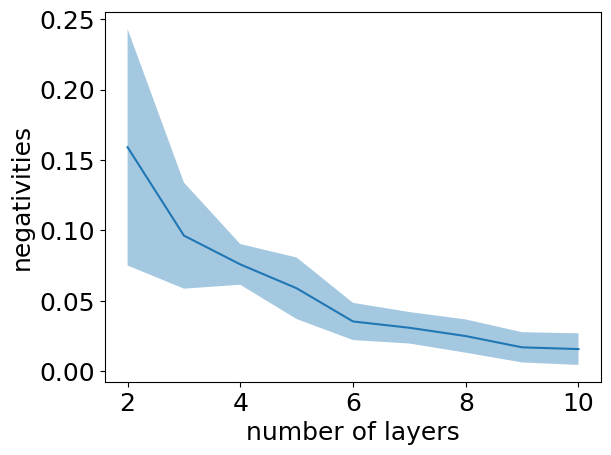}
    \includegraphics[width=0.24\linewidth]{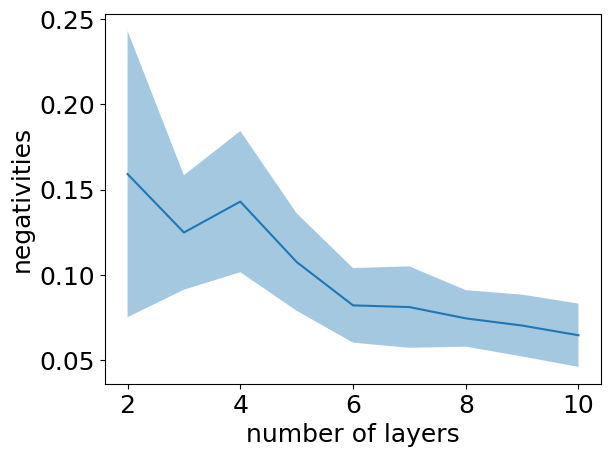}
    \caption{\textbf{Left to right:} Minimum balancedness; mean balancedness; minimum negativity; mean negativity across non-linear layers of the head as a function of the number of non-linear layers. Results are averaged over 10 runs (5 runs for each of the 2 hyperparameter setups), and the confidence band at 1 standard deviation is displayed.}
\label{fig:balancedness_negativity_vs_depth}
\end{figure}
In Figure~\ref{fig:balancedness_negativity_vs_depth}, we show the dependence of balancedness and non-linearity of the non-linear layers as a function of the depth of the non-linear part. In particular, we plot minimum balancedness and non-linearity across all layers, as well as mean balancedness and non-linearity. The depth of the non-linear part ranges from 4 to 12 (the first two layers are considered as backbone and not measured), the learning rate is either $0.001$ or $0.002$ (5 runs each), and the weight decay is $0.016$ divided by the total number of layers. Balancedness clearly improves with depth, both on average per layer and in the most balanced layer. Similarly, the non-negativity of the layer that least uses the ReLU clearly decreases, and a decrease in the mean negativity is also reported (although less pronounced). Thus, these results suggest that the network tends to use non-linearities less as the depth increases and, in addition, it becomes more balanced. 
We also note that, in order to fit the data, regardless of the depth, the last non-linear layer(s) always exhibit significant negativity, heavily relying on the ReLU. 

\section*{Acknowledgements}

M. M.\ and P. S.\ are funded by the European Union (ERC, INF$^2$, project number 101161364). Views and opinions expressed are however
those of the author(s) only and do not necessarily reflect those of the European Union or the
European Research Council Executive Agency. Neither the European Union nor the granting
authority can be held responsible for them.






\bibliography{bibliography.bib}
\bibliographystyle{iclr2025_conference}

\newpage

\appendix
\section{Additional Experiments}\label{app:experiments}
We complement the experiments from Section~\ref{sec:experiments} with additional numerical findings. We start by showing an analog of Figure~\ref{fig:single_hyperparam_demo} for MLP trained on MNIST to show that the behavior is robust with respect to the backbone and dataset. The results are shown in Figure~\ref{fig:single_hyperparam_demo_app}. The architecture is an MLP with 5 non-linear layers followed by 4 linear layers. We use weight decay of $0.0018$ and learning rate of $0.001$, training for 10000 epochs.
\begin{figure}
    \centering
    \includegraphics[width=0.24\linewidth]{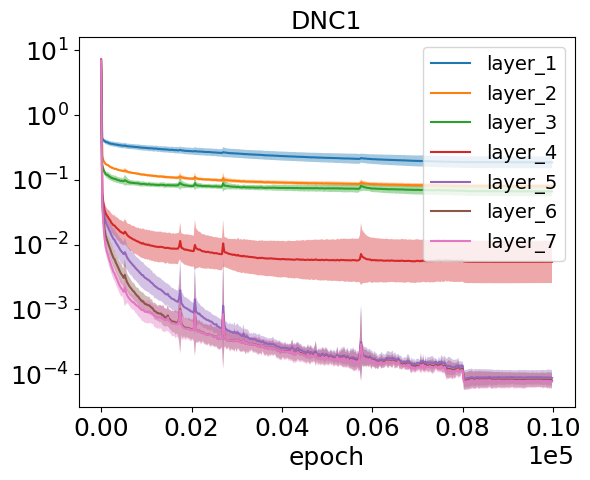}
    \includegraphics[width=0.24\linewidth]{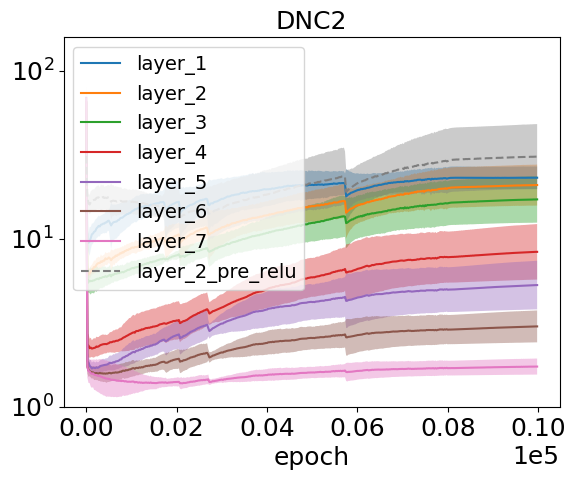}
    \includegraphics[width=0.24\linewidth]{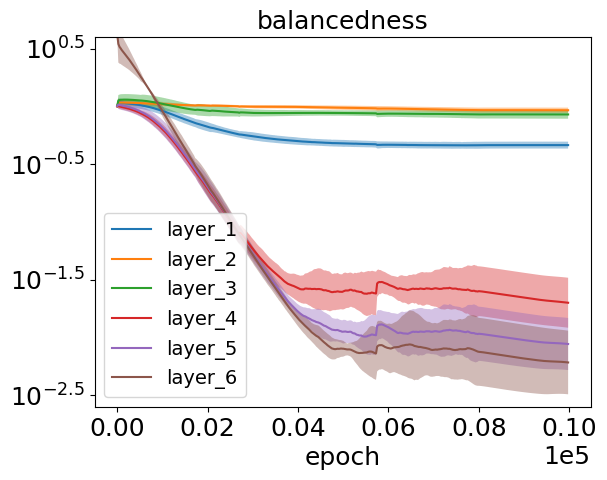}
    \includegraphics[width=0.24\linewidth]{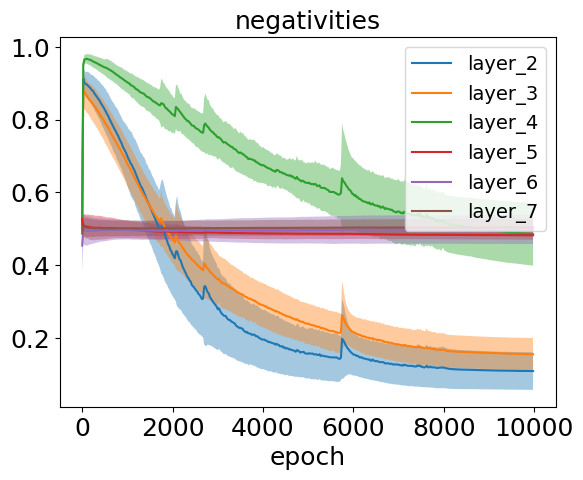}
    \includegraphics[width=0.24\linewidth]{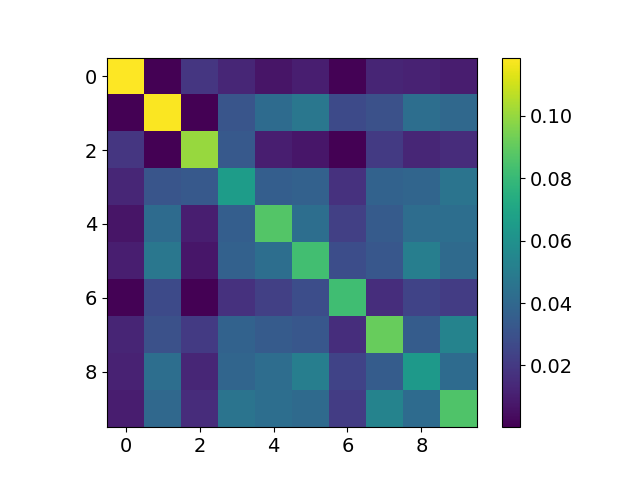}
    \includegraphics[width=0.24\linewidth]{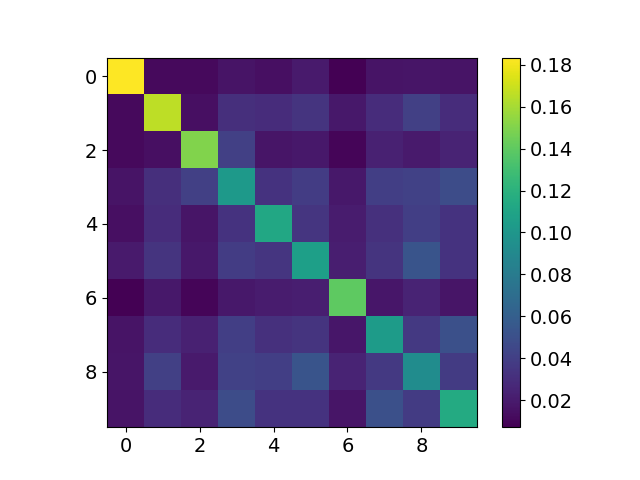}
    \includegraphics[width=0.24\linewidth]{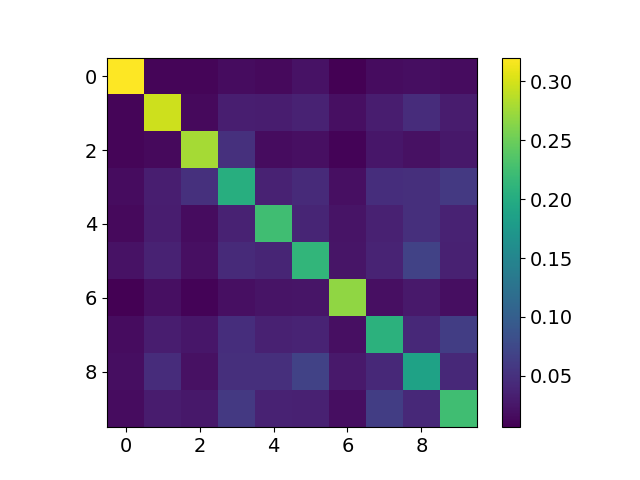}
    \includegraphics[width=0.24\linewidth]{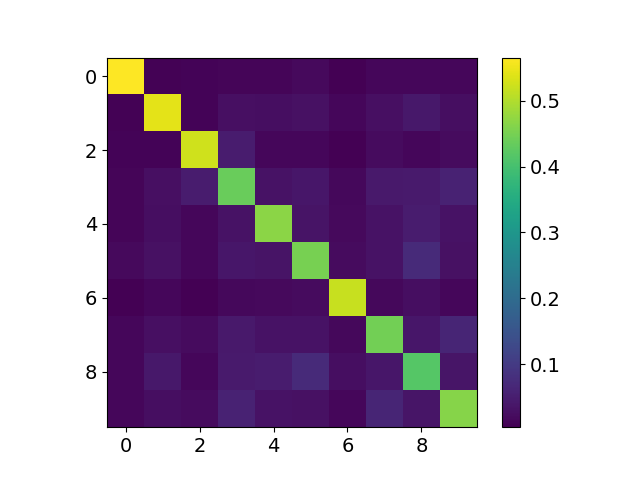}
    \caption{Last 7 layers of a 9-layer MLP trained on MNIST with weight decay $0.0018$ and learning rate $0.001$. \textbf{Top:} NC1s, NC2s, balancednesses and negativities, from left to right. Results are averaged over 5 runs, and the confidence band at 1 standard deviation is displayed. \textbf{Bottom:} Class-mean matrices of the last four layers (i.e., the linear head).}
    \label{fig:single_hyperparam_demo_app}
\end{figure}
The plot fully agrees with the one in Figure~\ref{fig:single_hyperparam_demo} with ResNet20 training on CIFAR10 in every qualitative aspect, and it only differs in the numerical values attained by the NC metrics. Therefore, the conclusion is that the NC is attained across different architectures and NC2 progressively improves as we get closer to the last layer of the DNN. 

Next, we extend Figure~\ref{fig:nc_vs_depth} with experiments on ResNet20 trained on MNIST, which are shown in Figure~\ref{fig:nc_vs_depth_app}. As above, the results match the interpretations discussed in Section~\ref{sec:experiments}, which proves the robustness of our findings across different achitectures, datasets and hyperparameter settings. 

\begin{figure}
    \centering
    \includegraphics[width=0.24\linewidth]{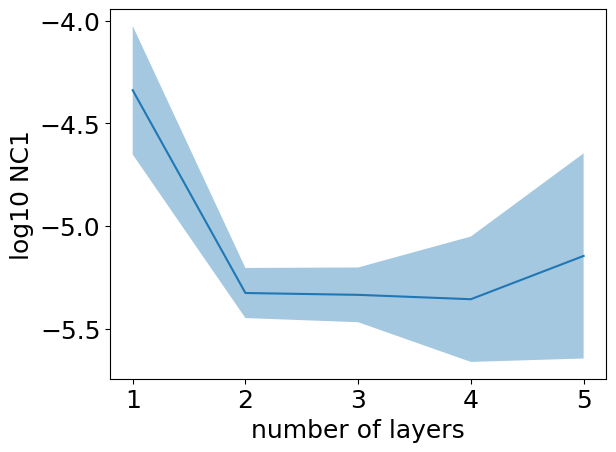}
    \includegraphics[width=0.24\linewidth]{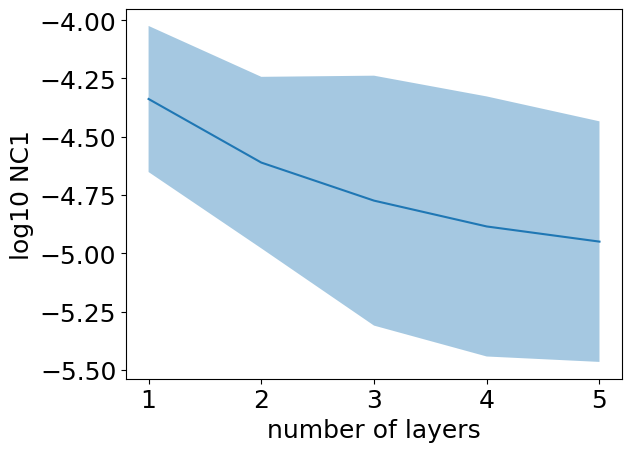}
    \includegraphics[width=0.24\linewidth]{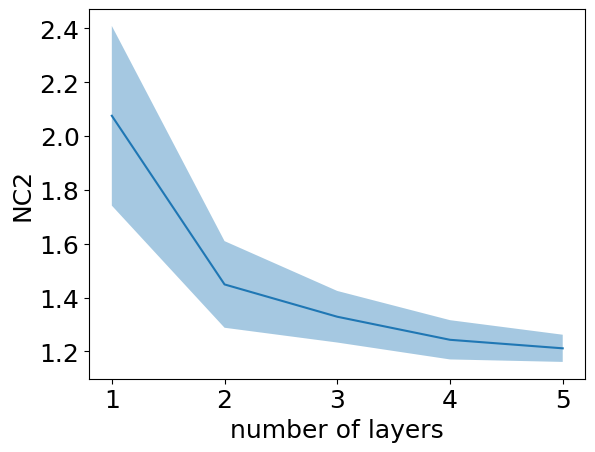}
    \includegraphics[width=0.24\linewidth]{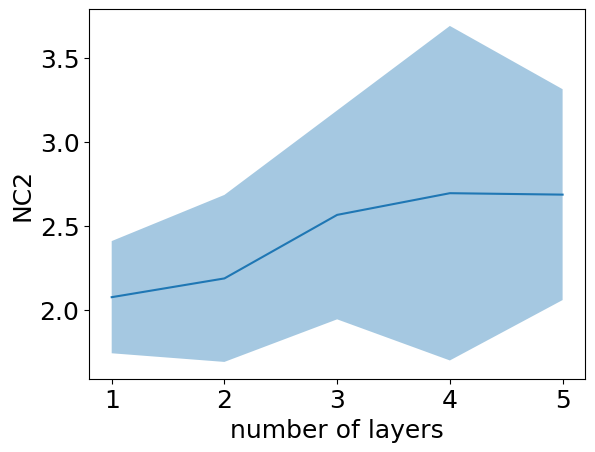}
    \caption{ResNet20 trained on MNIST with a deep linear head. \textbf{Left to right:} NC1 in the last layer; NC1 in the first layer of the linear head; NC2 in the last layer; NC2 in the first layer of the linear head. All plots are a function of the number of layers in the linear head. Results are based on 50 runs (5 runs for each of the 10 hyperparameter setups), and the confidence band at 1 standard deviation is displayed.}
    \label{fig:nc_vs_depth_app}
\end{figure}

\section{Deferred Proofs}\label{app:deferred}

\metaresult*

\begin{proof} \textbf{NC1:} We start by proving the claim \eqref{eq:NC1} on NC1. 
We have
\begin{equation}\label{eq:lbsK}
s_K(Y)-\epsilon_1 \le s_K(Y+(Z_L-Y))=s_K(Z_L)=s_K(W_LZ_{L-1})\le s_K(W_L)r.
\end{equation}
Denote $W_L^+$ the pseudoinverse of $W_L.$ Then, we have that $W_L^+W_L$ equals the projection $P$ on the row space of $W_L.$ We can now write $$Z_{L-1}=PZ_{L-1}+(I-P)Z_{L-1}=W_L^+Y+W_L^+(Z_L-Y)+(I-P)W_{L-1}Z_{L-2}.$$ Note that $$\left\Vert W_{L}^{+}(Z_{L}-Y)\right\Vert_{F}\leq\frac{\epsilon_{1}r}{s_K(Y)-\epsilon_1},$$ since $s_K(W_L)\ge (s_K(Y)-\epsilon_1)/r$ by \eqref{eq:lbsK}. Furthermore,  
\begin{align*}
\left\Vert (I-P)W_{L-1}Z_{L-2}\right\Vert_{F}^{2} &\le \left\Vert (I-P)W_{L-1}\right\Vert_{F}^{2}\left\Vert Z_{L-2}\right\Vert_{op}^{2} \le r^2 \text{tr}((I-P)W_{L-1}W_{L-1}^\top) \\ &= r^2 \text{tr}((I-P)(W_{L-1}W_{L-1}^\top-W_L^\top W_L)) \le r^2 n_{L-1} \epsilon_2.
\end{align*}
Putting these together, we have:
\begin{equation}\label{eq:bdZ}
\norm{Z_{L-1}-W_L^+Y}_F\le r\left(\frac{\epsilon_{1}}{s_K(Y)-\epsilon_1}+\sqrt{n_{L-1} \epsilon_2}\right)=\Psi(\epsilon_1, \epsilon_2, r).
\end{equation}

From now on, we will drop the index $L-1$ and treat everything without layer label as belonging to that layer (the membership to any other layer will be indexed). We have that
\begin{align*} 
\text{tr}(\Sigma_W)&=\text{tr}\left(\frac{1}{N}\sum_{c, i}
(z_{ci}-\mu_c)(z_{ci}-\mu_c)^\top\right)=\frac{1}{N}\sum_{c, i}
\left\Vert z_{ci}-\mu_c \right\Vert_2^2 \\ &\le \frac{1}{N}\sum_{c, i}
\left\Vert z_{ci}-(W_L^{+})_{c:}\right\Vert_2^2=\frac{1}{N}\left\Vert Z_{L-1}-W_L^{+}Y \right\Vert^2_F.    
\end{align*}
Furthermore,
\begin{align*}
\text{tr}(\Sigma_B)&=\text{tr}\left(\frac{1}{K}\sum_{c=1}^K (\mu_c-\mu_G)(\mu_c-\mu_G)^\top\right)=\frac{1}{K}\sum_{c=1}^K \left\Vert \mu_c - \mu_G \right\Vert_2^2  \\ &\ge \frac{\left\Vert W_L \right\Vert^2_{op}}{r^2}\frac{1}{K}\sum_{c=1}^K\left\Vert\mu_c-\mu_G\right\Vert_2^2\ge \frac{1}{Kr^2}\sum_{c=1}^K \left\Vert\mu_c^L-\mu_G^L\right\Vert_2^2.
\end{align*} 
Now we proceed by lower-bounding the last term: 
\begin{align*}
    \frac{1}{K}\sum_{c=1}^K \left\Vert\mu_c^L-\mu_G^L\right\Vert_2^2 &\ge \left(\frac{1}{K}\sum_{c=1}^K \left\Vert \mu_c^L-\mu_G^L\right\Vert_2\right)^2 \\ &\ge \left(\frac{1}{K}\sum_{c=1}^K \left\Vert \mu_c^Y-\mu_G^Y\right\Vert_2 - \frac{1}{K}\sum_{c=1}^K \left\Vert \mu_c^Y-\mu_c^L\right\Vert_2 - \left\Vert \mu_G^Y-\mu_G^L\right\Vert_2\right)^2,
\end{align*}
where $\mu_c^Y, \mu_G^Y$ are the class and global means of the label matrix $Y.$ A direct computation yields that $\frac{1}{K}\sum_{c=1}^K \left\Vert \mu_c^Y-\mu_G^Y\right\Vert = \sqrt{\frac{K-1}{K}}.$ Next we have:
\begin{align*}
\left\Vert \mu_G^Y-\mu_G^L\right\Vert_2 &=\left\Vert \frac{1}{N}\sum_{c, i}
z_{ci}^L-\frac{1}{N}\sum_{c, i}
z_{ci}^Y\right\Vert_2 \le \frac{1}{N}\sum_{c, i}
\left\Vert z_{ci}^L-z_{ci}^Y\right\Vert_2 \\ &\le \sqrt{\frac{1}{N}\sum_{c, i}
\left\Vert z_{ci}^L-z_{ci}^Y\right\Vert_2^2} = \frac{1}{\sqrt{N}}\left\Vert Z_L-Y\right\Vert_F \le \frac{\epsilon_1}{\sqrt{N}}. 
\end{align*}
Finally, for any fixed $c$ we have
\begin{align*}
\left\Vert \mu_c^Y-\mu_c^L \right\Vert_2 = \left\Vert \frac{1}{n}\sum_{i}
z_{ci}^Y-\frac{1}{n}\sum_{i}
z_{ci}^L\right\Vert_2 \le \frac{1}{n} \sum_{i}
\left\Vert z_{ci}^Y-z_{ci}^L\right\Vert_2.
\end{align*}
Therefore we get: 
\begin{align*}
\frac{1}{K}\sum_{c=1}^K \left\Vert \mu_c^Y-\mu_c^L\right\Vert_2 \le \frac{1}{N} \sum_{c,i}
\left\Vert z_{ci}^Y-z_{ci}^L\right\Vert_2 \le \frac{1}{\sqrt{N}} \left\Vert Z_L-Y\right\Vert_F \le \frac{\epsilon_1}{\sqrt{N}}.
\end{align*}
Upper bounding these terms in the above computation and dividing $\text{tr}(\Sigma_W)$ by $\text{tr}(\Sigma_B)$ yields \eqref{eq:NC1}. 

\textbf{Conditioning of $W_L$:}
We now prove the claim \eqref{eq:NC2} on the conditioning of $W_L$. We rely first on Lemma \ref{lem:approximate_balancedness_relate_weight_to_product} to relate the singular values of $W_{L}$ to the $L_2$-th root of those of $W_{L:L_1+1}$. We therefore obtain that
\begin{align*}
\frac{s_{1}(W_{L})^{2L_{2}}}{s_{K}(W_{L})^{2L_{2}}} & \leq\frac{s_{1}(W_{L:L_{1}+1})^{2}+\frac{L_2^2}{2}r^{2(L_2-1)}\epsilon_2}{s_{K}(W_{L:L_{1}+1})^{2}-\frac{L_2^2}{2}r^{2(L_2-1)}\epsilon_2}\\
 & =\frac{s_{1}(W_{L:L_{1}+1})^{2}}{s_{K}(W_{L:L_{1}+1})^{2}}+\frac{s_{K}(W_{L:L_{1}+1})^{2}\frac{L_2^2}{2}r^{2(L_2-1)}\epsilon_2+s_{1}(W_{L:L_{1}+1})^{2}\frac{L_2^2}{2}r^{2(L_2-1)}\epsilon_2}{\left(s_{K}(W_{L:L_{1}+1})^{2}-\frac{L_2^2}{2}r^{2(L_2-1)}\epsilon_2\right)s_{K}(W_{L:L_{1}+1})^{2}}.
\end{align*}
We have that 
\[
s_{K}(Z_{L})\leq s_{K}(W_{L:L_{1}+1})\left\Vert Z_{L_{1}}\right\Vert _{op}\leq s_{K}(W_{L:L_{1}+1})\left\Vert X\right\Vert _{op}r^{L_{1}},
\]
and therefore
\[
s_{K}(W_{L:L_{1}+1})\geq\frac{s_{K}(Z_{L})}{\left\Vert X\right\Vert _{op}r^{L_{1}}}\geq\frac{s_{K}(Y)-\epsilon_{1}}{\left\Vert X\right\Vert _{op}r^{L_{1}}}.
\]
This gives that 
\[
\frac{s_{1}(W_{L})^{2L_{2}}}{s_{K}(W_{L})^{2L_{2}}}\leq\frac{s_{1}(W_{L:L_{1}+1})^{2}}{s_{K}(W_{L:L_{1}+1})^{2}}+\frac{1+\kappa(W_{L:L_{1}+1})^{2}}{\left(\frac{(s_{K}(Y)-\epsilon_{1})^{2}}{\left\Vert X\right\Vert _{op}^{2}r^{2L_{1}}}-\frac{L_2^2}{2}r^{2(L_2-1)}\epsilon_2\right)}\frac{L_2^2}{2}r^{2(L_2-1)}\epsilon_2.
\]
Thus, the following chain of inequalities gives \eqref{eq:NC2}: 
\begin{align*}
\kappa(W_{L}) & \leq\left(\kappa(W_{L:L_{1}+1})^{2}(1+\epsilon)+\epsilon\right)^{\frac{1}{2L_{2}}}\\
 & \leq\kappa(W_{L:L_{1}+1})^{\frac{1}{L_{2}}}(1+\epsilon)^{\frac{1}{2L_2}}+{\kappa(W_{L:L_{1}+1})^{\frac{1}{L_{2}}-1}}(1+\epsilon)^{\frac{1}{2L_2}-1}\epsilon\\
 & \leq\kappa(W_{L:L_{1}+1})^{\frac{1}{L_{2}}}(1+\epsilon)^{\frac{1}{2L_2}}+{\kappa(W_{L:L_{1}+1})^{\frac{1}{L_{2}}-1}}\epsilon,
\end{align*}
using the concavity of the $2L_{2}$-th root. 

\textbf{NC(2+3):} 
We start with the proof of the claim \eqref{eq:NC3} on NC3, as the derivations will be used in the proof of the claim \eqref{eq:NC2*} on NC2 later. Since the cosine similarity does not depend on the scale of the involved matrices, we perform the following rescaling: denoting $\alpha=\left\Vert W_L\right\Vert$, we define $W_L'=\frac{W_L}{\alpha}$ and $Z_{L-1}'=\alpha Z_{L-1}.$ Similarly to before we denote $(z')^{(L-1)}_{ci}$ the $i$-th sample of the $c$-th class in the matrix $
Z_{L-1}'.$ Then, we write
\begin{align*}
\text{NC3}(Z_{L-1}', W_L')&=\frac{1}{N}\sum_{c,i}
\cos((z')^{(L-1)}_{ci}, (W_{L}')_{c:})=\frac{1}{N}\sum_{c,i}
\frac{\left\langle (z')^{(L-1)}_{ci}, (W'_{L})_{c:}\right\rangle}{\left\Vert (z')^{(L-1)}_{ci}\right\Vert_2 \left\Vert (W'_{L})_{c:}\right\Vert_2} \\ &\ge \frac{\sum_{c,i}
\left\langle (z')^{(L-1)}_{ci}, (W'_L)_{c:}\right\rangle}{N\kappa(W_L)(1+\epsilon_1)} =\frac{2\left\langle Z_{L-1}', (W'_L)^\top Y\right\rangle}{2N\kappa(W_L)(1+\epsilon_1)}\\
&=\frac{\norm{Z'_{L-1}}_F^2+\norm{(W'_L)^\top Y}_F^2-\norm{Z'_{L-1}-(W'_L)^\top Y}_F^2}{2N\kappa(W_L)(1+\epsilon_1)}.
\end{align*}
Here, the first inequality follows from upper bounding $\norm{(W'_L)_{c:}}_2$ trivially by 1 and $$\norm{(z')^{(L-1)}_{ci}}_2\le \kappa(W_L)\norm{z^{(L)}_{ci}}_2\le \kappa(W_L)\left(\norm{y_c-z^{(L)}_{ci}}_2+\norm{y_c}_2\right)\le \kappa(W_L)(1+\epsilon_1).$$ Now, 
$$\norm{Z'_{L-1}}_F^2\ge \norm{Z_L}_F^2\ge \left(\norm{Y}_F-\norm{Z_L-Y}_F\right)^2=(\sqrt{N}-\epsilon_1)^2.$$ Furthermore, we readily have that $\norm{(W'_L)^\top Y}_F\ge \frac{\sqrt{N}}{\kappa(W_L)}.$  Finally, we have:
\begin{align*}
\norm{Z'_{L-1}-(W_L')^\top Y}_F\le \norm{Z'_{L-1}-(W'_L)^{+}Y}_F+\norm{(W'_L)^+Y-(W'_L)^\top Y}_F. 
\end{align*} 
From \eqref{eq:bdZ}, we obtain that  \begin{align}\label{eq:big_triangle_ineq}
    \norm{Z'_{L-1}-(W_L')^{+}Y}_F\le r^2\left(\frac{\epsilon_{1}}{s_K(Y)-\epsilon_1}+\sqrt{n_{L-1} \epsilon_2}\right)=r\Psi(\epsilon_1, \epsilon_2, r).
\end{align} 
Finally, we proceed with upper bounding $\norm{(W'_L)^+Y-(W'_L)^\top Y}_F$, which can be done via a sandwich bound on the singular values using the conditioning number. This gives 
$$\norm{(W'_L)^+Y-(W'_L)^\top Y}_F\le \sqrt{K}\left( \kappa(W_L)-\frac{1}{\kappa(W_L)}\right)\le \sqrt{K}(\kappa(W_L)^2-1).$$ Putting all the obtained bounds together, we get the desired bound \eqref{eq:NC3} on NC3. 

Finally, to pass from the bound \eqref{eq:NC2} on $\kappa(W_{L})$ to the bound \eqref{eq:NC2*} on NC2, we use the inequality in~\eqref{eq:big_triangle_ineq} obtained in the proof of NC3. Note that obtaining a bound on $\kappa(\Bar{Z}'_{L-1})$ is equivalent to obtaining a bound on $\kappa(\Bar{Z}_{L-1})$ since multiplying by a scalar does not change the condition number. Using~\eqref{eq:big_triangle_ineq} we get: \begin{align*}
    \norm{\Bar{Z}'_{L-1}-(W'_L)^+}_F&=\norm{Z'_{L-1}Y^+-(W'_L)^+YY^+}_F \le \norm{Y^+}_{op} \norm{Z'_{L-1}-(W'_L)^+ Y}_F \\ &\le \frac{r\Psi(\epsilon_1, \epsilon_2, r)}{s_K(Y)}.
\end{align*}
As $\|(W'_L)^+\|_{op}=\kappa(W_L)$ and $s_K((W'_L)^+)=1$, we conclude that
$$\kappa(\Bar{Z}'_{L-1})\le \frac{\kappa(W_L)+r\Psi(\epsilon_1, \epsilon_2, r)(s_K(Y))^{-1}}{1-r\Psi(\epsilon_1, \epsilon_2, r)(s_K(Y))^{-1}},$$ which gives the desired bound in \eqref{eq:NC2*}. 

\end{proof}

\PL*
\begin{proof}
Let $\theta \in B(\theta_{0}, r_0)$. Then, the following chain of inequalities holds:
\begin{equation}\label{eq:chainPL}
\begin{split}    
\left\Vert \nabla C_0(\theta)+\lambda\theta\right\Vert_2^{2} & \geq\left(\left\Vert \nabla C_0(\theta)\right\Vert_2 -\lambda\left\Vert \theta\right\Vert_2 \right)^{2}\\
 & \geq\left(\sqrt{\frac{\alpha}{2}C_0(\theta)}-\lambda\left\Vert \theta\right\Vert_2 \right)^{2}\\
 & =\left(\sqrt{\frac{\alpha}{2}C_{\lambda}(\theta)-\frac{\alpha\lambda}{4}\left\Vert \theta\right\Vert_2 ^{2}}-\lambda\left\Vert \theta\right\Vert_2 \right)^{2}\\
 & \geq\left(\sqrt{\frac{\alpha}{2}C_{\lambda}(\theta)}-\left(\lambda+\sqrt{\frac{\alpha\lambda}{4}}\right)\left\Vert \theta\right\Vert_2 \right)^{2}\\
 & \geq\frac{\alpha}{4}C_{\lambda}(\theta)-\lambda\left(\sqrt{\frac{\alpha}{4}}+\sqrt{\lambda}\right)^{2}\left\Vert \theta\right\Vert_2^{2}\\
 & \geq\frac{\alpha}{4}\left(C_{\lambda}(\theta)-\lambda\left(1+\sqrt{\frac{4\lambda}{\alpha}}\right)^{2}\left(\left\Vert \theta_{0}\right\Vert_2 +r_0\right)^{2}\right).
\end{split}
\end{equation}
Here,
in the second line we use that $\sqrt{\frac{\alpha}{2}C_0(\theta)}\ge \lambda\left\Vert \theta\right\Vert_2$, which follows from $C_{\lambda}(\theta)\ge \lambda m_{\lambda}$ (otherwise, the claim is trivial); in the fourth line we use that $\sqrt{a-b}\geq\sqrt{a}-\sqrt{b}$ for $a\ge b$; in the fifth line we use
that $(a-b)^{2}\geq\frac{a^{2}}{2}-b^{2}$ for all $a, b$; and in the sixth line we use that $\theta$ is in the ball $B(\theta_{0}, r_0)$. As the LHS of \eqref{eq:chainPL} equals $\left\Vert \nabla C_{\lambda}(\theta))\right\Vert_2 ^{2}$, this proves \eqref{eq:shiftedPL}.

Next, let $(\theta_k)_{k\in\mathbb N}$ be the GD trajectory. 
Pick $k$ s.t. $\theta_k\in B(\theta_0, r_0)$. Then,
\begin{align*}
& C_{\lambda}(\theta_{k+1})-C_{\lambda}(\theta_k)  =-\eta\int_{0}^{1}\langle \nabla C_{\lambda}(\theta_k-s\eta\nabla C_{\lambda}(\theta_k)), \nabla C_{\lambda}(\theta_k)\rangle ds\\
 & =-\eta\left\Vert \nabla C_{\lambda}(\theta_k)\right\Vert_2 ^{2}+\eta\int_{0}^{1}\langle \nabla C_\lambda(\theta_k-s\eta\nabla C_{\lambda}(\theta_k))-\nabla C_\lambda(\theta_k), \nabla C_{\lambda}(\theta_k)\rangle ds\\
 & \leq-\eta\left\Vert \nabla C_{\lambda}(\theta_k)\right\Vert_2 ^{2}+\eta\int_{0}^{1}\left\Vert \nabla C_{\lambda}(\theta_k)\right\Vert_2 \left\Vert \nabla C_\lambda(\theta_k-s\eta\nabla C_{\lambda}(\theta_k))-\nabla C_\lambda(\theta_k)\right\Vert_2 ds\\
 & \leq-\eta\left\Vert \nabla C_{\lambda}(\theta_k)\right\Vert_2 ^{2}+\eta^{2}\beta_1\left\Vert \nabla C_{\lambda}(\theta_k)\right\Vert_2 ^{2}\\
 & =-\eta(1-\eta\beta_1)\left\Vert \nabla C_{\lambda}(\theta_k)\right\Vert_2 ^{2}\\
 & \le -\frac{\eta}{2}\left\Vert \nabla C_{\lambda}(\theta_k)\right\Vert_2 ^{2}\\
 & \leq-\eta\frac{\alpha}{8}\left(C_{\lambda}(\theta_k)-\lambda m_{\lambda}\right).
\end{align*}
Here, in the fourth line we use that the gradient is $\beta_1$-Lipschitz in $B(\theta_0, r_0)$; in the sixth line we use that $\eta<1/(2\beta_1)$; and in the last line we use \eqref{eq:shiftedPL}. Thus, as long as $\theta_j\in B(\theta_0, r_0)$ for all $j\in [k]$, by iterating the argument above, we have
\[
C_{\lambda}(\theta_k)-\lambda m_{\lambda}\leq\left(C_{\lambda}(\theta_{0})-\lambda m_{\lambda}\right)\left(1-\eta\frac{\alpha}{8}\right)^{k}.
\]
This readily implies that, by letting $k_1$ be the first index $k$ s.t.\ $C_{\lambda}(\theta_{k})\leq2\lambda m_{\lambda}$, \eqref{eq:Tbound} holds. 
Finally, the distance $\left\Vert \theta_{k_1}-\theta_{0}\right\Vert_2 $
is upper bounded by 
\begin{align*}
\sum_{k=0}^{k_1-1}\eta &\left\Vert \nabla C_{\lambda}(\theta_{k})\right\Vert_2   =\eta\sum_{k=0}^{k_1-1}\frac{\left\Vert \nabla C_{\lambda}(\theta_{k})\right\Vert_2 ^{2}}{\left\Vert \nabla C_{\lambda}(\theta_{k})\right\Vert_2 }\\
 & \leq\frac{4}{\sqrt{\alpha}}\sum_{k=0}^{k_1-1}\frac{C_{\lambda}(\theta_{k})-C_{\lambda}(\theta_{k+1})}{\sqrt{C_{\lambda}(\theta_{k})-\lambda m_{\lambda}}}\\
 & =\frac{4}{\sqrt{\alpha}}\sum_{k=0}^{k_1-2}\frac{C_{\lambda}(\theta_{k})-C_{\lambda}(\theta_{k+1})}{\sqrt{C_{\lambda}(\theta_{k})-\lambda m_{\lambda}}}+\frac{4}{\sqrt{\alpha}}\frac{C_{\lambda}(\theta_{k_1-1})-C_{\lambda}(\theta_{k_1})}{\sqrt{C_{\lambda}(\theta_{k_1-1})-\lambda m_{\lambda}}}\\
 & \leq\frac{8}{\sqrt{\alpha}}\sum_{k=0}^{k_1-2}\frac{C_{\lambda}(\theta_{k})-C_{\lambda}(\theta_{k+1})}{\sqrt{C_{\lambda}(\theta_{k+1})-\lambda m_{\lambda}}+\sqrt{C_{\lambda}(\theta_{k})-\lambda m_{\lambda}}}+\frac{8}{\sqrt{\alpha}}\frac{C_{\lambda}(\theta_{k_1-1})-\lambda m_\lambda}{\sqrt{C_{\lambda}(\theta_{k_1-1})-\lambda m_{\lambda}}}\\
 & =\frac{8}{\sqrt{\alpha}}\sum_{k=0}^{k_1-2}\left(\sqrt{C_{\lambda}(\theta_{k})-\lambda m_{\lambda}}-\sqrt{C_{\lambda}(\theta_{k+1})-\lambda m_{\lambda}}\right)+\frac{8}{\sqrt{\alpha}}\sqrt{C_{\lambda}(\theta_{k_1-1})-\lambda m_{\lambda}}\\
 & =\frac{8}{\sqrt{\alpha}}\sqrt{C_{\lambda}(\theta_{0})-\lambda m_\lambda} \leq \frac{8}{\sqrt{\alpha}}\sqrt{C_{\lambda}(\theta_{0})},
\end{align*}
which concludes the proof. 
\end{proof}

\mainGD*
\begin{proof}
By Lemma 4.1 in \citep{QuynhMarco2020}, the loss $C_0(\theta)$ satisfies the $\alpha$-PL inequality with $$\alpha=4\gamma^{L-2}\svmin{Z_1}\prod_{p=3}^L \svmin{W_p},$$ where we have also used that $\sigma'$ is lower bounded by $\gamma$ by Assumption \ref{ass:act}. Thus, by taking $r_0=\frac{1}{2}\min(\lambda_F, \min_{\ell\in \{3, \ldots, L\}} \lambda_\ell)$, we have that, for all $\theta\in B(\theta_0, r_0)$, $C_0(\theta)$ satisfies the $\alpha$-PL inequality with $\alpha=2^{-(L-3)}\gamma^{L-2}\lambda_F\lambda_{3\to L}$.

By using Assumption \ref{ass:init} and that $\lambda$ is upper bounded by $\frac{2C_0(\theta_0)}{\norm{\theta_0}_2^2}$ in \eqref{eq:ublambdaeta}, one can readily verify that $r_0\geq 8\sqrt{C_{\lambda}(\theta_{0})/\alpha}$. Furthermore, by Lemma \ref{lem:Lipschitzness_gradient_bounded_weights}, we have that $\nabla C_0(\theta)$ is $\beta_1$-Lipschitz for all $\theta\in B(\theta_0, r_0)$. Hence, we can apply Proposition \ref{prop:PL}, which gives that, for some $k_1$ upper bounded in \eqref{eq:Tbound},  $$C_{\lambda}(\theta_{k_1})\leq2\lambda m_{\lambda}\le \epsilon_1^2,$$ where the last inequality uses again the upper bound on $\lambda$ in \eqref{eq:ublambdaeta}.

With gradient flow, the regularized loss would only
decrease further after $k_1$ steps. Since we are working with gradient descent,
we simply need to assume that the learning rate is small enough to
guarantee a decreasing loss. If the gradient $\nabla C_{\lambda}(\theta)$
is $\beta_2$-Lipschitz, then
\begin{align*}
C_{\lambda}(\theta_{k+1})&-C_{\lambda}(\theta_k)  =C_{\lambda}(\theta_k-\eta\nabla C_{\lambda}(\theta_k))-C_{\lambda}(\theta_k)\\
 & =-\eta\int_{0}^{1}\left\langle \nabla C_{\lambda}(\theta_k-s\eta\nabla C_{\lambda}(\theta_k)),\nabla C_{\lambda}(\theta_k)\right\rangle ds\\
 & \leq-\eta\left\Vert \nabla C_{\lambda}(\theta_k)\right\Vert_2^{2}\\ &\hspace{1.1em}+\eta\left\Vert \nabla C_{\lambda}(\theta_k)\right\Vert_2 \max_{s\in[0,1]}\left\Vert \nabla C_{\lambda}(\theta_k)-\nabla C_{\lambda}(\theta_k-s\eta\nabla C(\theta_k))\right\Vert_2 \\
 & \leq-\eta\left\Vert \nabla C_{\lambda}(\theta_k)\right\Vert_2^{2}+\eta^{2}\beta_2\left\Vert \nabla C_{\lambda}(\theta_k)\right\Vert _2^{2}\\
 & =-\eta(1-\eta\beta_2)\left\Vert \nabla C_{\lambda}(\theta_k)\right\Vert ^{2}_2,
\end{align*}
which is non-positive as long as $\eta\leq 1/\beta_2$. Furthermore, 
as long as the regularized loss is decreasing, the parameter norm is bounded by $\epsilon_1\sqrt{2/\lambda}$. Lemma \ref{lem:Lipschitzness_gradient_bounded_weights} then implies that the gradient is $\beta_2 = 5N\beta b^{3}\max\left(1, \epsilon_1^{3L}\left(\frac{2}{\lambda}\right)^{3L/2}\right)L^{5/2}$-Lipschitz and, therefore, $\eta\leq 1/\beta_2$ holds by \eqref{eq:ublambdaeta}. This allows us to conclude that, for all $k$ satisfying the lower bound in \eqref{eq:ublambdaeta}, $C_{\lambda}(\theta_k)\leq\epsilon_{1}^{2}$, hence the network achieves approximate interpolation, i.e.,
\begin{equation}\label{eq:appint}    
\|Z_L^k-Y\|_F\le \epsilon_1\sqrt{2}.
\end{equation}

Next, we show approximate balancedness. To do so, for $\ell\in \{L_1+2, \ldots, L-1\}$, we define $$T_\ell^k:=(W_{\ell+1}^k)^\top\cdots (W_{L}^k)^\top  (Z_L^k-Y) (Z_{L_1}^k)^\top (W_{L_1+1}^k)^\top\cdots (W_{\ell-1}^k)^\top.$$ 
Then, we have 
\begin{align*}
    W_\ell^{k+1} &(W_\ell^{k+1})^\top=((1-\eta\lambda)W_\ell^k-\eta T_\ell^k)((1-\eta\lambda)W_\ell^k-\eta T_\ell^k)^\top\\
    &=(1-\eta\lambda)^2W_\ell^k (W_\ell^k)^\top -(1-\eta\lambda)\eta(W_\ell^k (T_\ell^k)^\top + T_\ell^k (W_\ell^k)^\top)+ \eta^2 T_\ell^k (T_\ell^k)^\top.
\end{align*}
Similarly, 
\begin{align*}
 &   (W_{\ell+1}^{k+1})^\top W_{\ell+1}^{k+1}
    \\&=(1-\eta\lambda)^2(W_{\ell+1}^k)^\top W_{\ell+1}^k -(1-\eta\lambda)\eta((W_{\ell+1}^k)^\top T_{\ell+1}^k + (T_{\ell+1}^k)^\top W_{\ell+1}^k)+ \eta^2 (T_{\ell+1}^k)^\top T_{\ell+1}^k.
\end{align*}
Let us define
$$
D_\ell^k= (W_{\ell+1}^k)^\top W_{\ell+1}^k
- W_{\ell}^k (W_{\ell}^k)^\top.$$ 
Since $T_\ell^k (W_\ell^k)^\top= (W_{\ell+1}^k)^\top T_{\ell+1}^k$ and $W_\ell^k (T_\ell^k)^\top = (T_{\ell+1}^k)^\top W_{\ell+1}^k$, we have
\begin{align*}
   D_{\ell}^{k+1}=(1-\eta\lambda)^2 D_{\ell}^k+\eta^2((T_{\ell+1}^k)^\top T_{\ell+1}^k-T_\ell^k (T_\ell^k)^\top).
\end{align*}
Recall that, for all $k$ lower bounded in \eqref{eq:ublambdaeta}, $\|Z_L^k-Y\|_F\le \epsilon_1\sqrt{2}$ and $\|W_\ell^k\|_F \leq\|\theta^k\|_2\le \epsilon_1\sqrt{2/\lambda}$, which also implies that $\|Z_{L_1}\|_{F}\le \left(\epsilon_1\sqrt{2/\lambda}\right)^{L_1}\|X\|_{op}$.
Thus,
\begin{align*}
    \|D_{\ell}^{k+1}\|_{op}&\le (1-\eta\lambda)^2\|D_{\ell}^k\|_{op} +\eta^2(\|T_\ell^k\|_{op}^2+\|T_{\ell+1}^k\|_{op}^2)\\
    &\le (1-\eta\lambda)^2\|D_{\ell}^k\|_{op} + \eta^2 \|Z_{L_1}^k\|_{op}^2\|Z_L^k-Y\|_{op}^2(\prod_{j\neq \ell}\|W_j^k\|_{op}^2+\prod_{j\neq \ell+1}\|W_j^k\|_{op}^2)\\
    &\le (1-\eta\lambda)^2\|D_{\ell}^k\|_{op} + 4\eta^2\epsilon_1^2 \left(\frac{2\epsilon_1^2}{\lambda}\right)^{L_1+L-1}\|X\|_{op}^2.
\end{align*}
By using the upper bounds $\eta\le \frac{1}{2\lambda}$ and $\eta\leq \left(\frac{\lambda}{2\epsilon_1^2}\right)^{L_1+L}\frac{\epsilon_2}{4\|X\|_{op}^2}$ in \eqref{eq:ublambdaeta}, we have that:
\begin{itemize}
    \item if $\|D_\ell^k\|_{op}\ge \epsilon_2$, then $\|D_\ell^{k+1}\|_{op}\le (1-\eta\lambda)\|D_\ell^k\|_{op}$;
    \item if $\|D_\ell^k\|_{op}\le \epsilon_2$, then $\|D_\ell^{k+1}\|_{op}\le(1-\eta\lambda)\epsilon_2\le \epsilon_2$.
\end{itemize}
This implies that, for all $\bar k\ge 0$ and $k_1\ge \left\lceil \frac{\log\frac{\lambda m_{\lambda}}{C_{\lambda}(\theta_{0})-\lambda m_{\lambda}}}{\log(1-\eta\frac{\alpha}{8})}\right\rceil$,
\[
\|D_{\ell}^{k_{1}+\bar k}\|_{op}\le\max((1-\eta\lambda)^{\bar k}\|D_{\ell}^{k_{1}}\|_{op},\epsilon_{2}).
\]
Note that
$$
\|D_{\ell}^{k_{1}}\|_{op}\le \|W_{\ell+1}^{k_1}\|_{op}^2+\|W_{\ell}^{k_1}\|_{op}^2\le \frac{4\epsilon_1^2}{\lambda},
$$
which allows us to conclude that, for all $k$ lower bounded in \eqref{eq:ublambdaeta}, $$\|D_\ell^k\|_2\le \epsilon_2.$$
Finally, we have the following bounds on the representations and weights at step $k$:
\begin{align*}
    \left\Vert Z_{L-2}^k\right\Vert _{op} &\le \left(\epsilon_1\sqrt{\frac{2}{\lambda}}\right)^{L-2}\|X\|_{op}, \\
    \left\Vert Z_{L-1}^k\right\Vert _{op}&\le \left(\epsilon_1\sqrt{\frac{2}{\lambda}}\right)^{L-1}\|X\|_{op}, \\
    \left\Vert W_{\ell}^k\right\Vert _{op}&\leq \epsilon_1\sqrt{\frac{2}{\lambda}}, \,\,\,\mbox{ for } \ell\in \{L_1+1, \ldots, L\}.
\end{align*}
Hence, an application of Theorem \ref{th:NC1} proves the desired result. \end{proof}

\goodlossimpliesnctwo*
\begin{proof}
Let us split the parameter norm into the contribution from the nonlinear
and linear layers:
\[
\left\Vert \theta\right\Vert_2 ^{2}=\left\Vert \theta_{nonlin}\right\Vert_2 ^{2}+\left\Vert \theta_{lin}\right\Vert_2 ^{2}.
\]
We first lower bound both parts in terms of the product of
the linear part $W_{L:L_{1}+1}$. This then allows
us to bound the conditioning of $W_{L:L_{1}+1}$. We start with the nonlinear part:
$$\left\Vert Z_{L_{1}}\right\Vert _{F}\leq\left\Vert X\right\Vert _{op}\prod_{\ell=1}^{L_{1}}\left\Vert W_{\ell}\right\Vert _{F}\leq\left\Vert X\right\Vert _{op}\left(\frac{1}{L_{1}}\left\Vert \theta_{nonlin}\right\Vert_2 ^{2}\right)^{\frac{L_{1}}{2}},$$
which implies that 
\[
\left\Vert \theta_{nonlin}\right\Vert_2 ^{2}\geq L_{1}\left(\frac{\left\Vert Z_{L_{1}}\right\Vert _{F}}{\left\Vert X\right\Vert _{op}}\right)^{\frac{2}{L_{1}}}\geq L_{1}\left(\frac{\left\Vert Z_{L_{1}}\right\Vert _{op}}{\left\Vert X\right\Vert _{op}}\right)^{\frac{2}{L_{1}}}.
\]
Since $\left\Vert Y-W_{L:L_{1}+1}Z_{L_{1}}\right\Vert _{F}\leq\epsilon_{1}$,
we have 
\[s_K(Y) \leq s_K(W_{L:L_1+1} Z_{L_1}) + \epsilon_1 \leq s_K(W_{L:L_1+1}) \left\Vert Z_{L_{1}}\right\Vert _{op} + \epsilon_1,
\]
so that $\left\Vert Z_{L_{1}}\right\Vert _{op}\geq\frac{s_K(Y)-\epsilon_{1}}{s_K(W_{L:L_{1}+1})}$ 
and thus
\begin{equation}    \label{eq:res1}
\left\Vert \theta_{nonlin}\right\Vert_2 ^{2}\geq L_{1}\left(\frac{s_K(Y)-\epsilon_{1}}{s_K(W_{L:L_{1}+1})\left\Vert X\right\Vert _{op}}\right)^{\frac{2}{L_{1}}}.
\end{equation}
Next, for the linear part, we know from Theorem 1 of \cite{dai_2021_repres_cost_DLN} that 
\[
\min_{(W_\ell)_{\ell=L_1+1}^L : A = W_{L:L_1+1}} \sum_{\ell=L_1+1}^L \| W_\ell \|_F^2 = (L-L_1) \sum_{i=1}^{\mathrm{Rank}(A)} s_i(A)^\frac{2}{L-L_1},
\]
so that 
\begin{equation}    \label{eq:res2}
\left\Vert \theta_{lin}\right\Vert_2 ^{2}\geq (L-L_1) \sum_{i=1}^{\mathrm{Rank}(W_{L:L_{1}+1})} s_i(W_{L:L_{1}+1})^\frac{2}{L-L_1} \geq (L-L_1)K+2\log\left|W_{L:L_{1}+1}\right|_{+},
\end{equation}
where we used the fact that $x^\frac{2}{L-L_1} \geq 1 + \frac{2}{L-L_1}\log x$.

We also have following bound on the condition number $\kappa(W_{L:L_{1}+1})$:
\begin{equation}    \label{eq:res3}
\kappa(W_{L:L_{1}+1})=\frac{s_{1}(W_{L:L_{1}+1})}{s_{K}(W_{L:L_{1}+1})}\leq\prod_{i=1}^{K}\frac{s_{i}(W_{L:L_{1}+1})}{s_{K}(W_{L:L_{1}+1})}=\frac{\left|W_{L:L_{1}+1}\right|_{+}}{s_K(W_{L:L_{1}+1})^{K}}.
\end{equation}

By combining \eqref{eq:res1}, \eqref{eq:res2}, and \eqref{eq:res3} (after applying the log on both sides of \eqref{eq:res3}), we obtain 
\begin{align*}
\left\Vert \theta\right\Vert_2 ^{2} & \geq L_{1}\left(\frac{s_K(Y)-\epsilon_{1}}{s_K(W_{L:L_{1}+1})\left\Vert X\right\Vert _{op}}\right)^{\frac{2}{L_{1}}}+(L-L_1)K+2\log\left|W_{L:L_{1}+1}\right|_{+}\\
 & \geq L_{1}\left(\frac{s_K(Y)-\epsilon_{1}}{s_K(W_{L:L_{1}+1})\left\Vert X\right\Vert _{op}}\right)^{\frac{2}{L_{1}}}+(L-L_1)K+2\log\kappa(W_{L:L_{1}+1})+2K\log s_K (W_{L:L_{1}+1}).
\end{align*}
The above is lower bounded by the minimum over all possible choices of $s_K(W_{L:L_{1}+1})$.
This minimum would be attained at $s_K(W_{L:L_{1}+1})=K^{-\frac{L_{1}}{2}}\frac{s_K(Y)-\epsilon_{1}}{\left\Vert X\right\Vert _{op}}$,
thus leading to the lower bound 
\[
\left\Vert \theta\right\Vert_2 ^{2}\geq LK+2\log\kappa(W_{L:L_{1}+1})-L_{1}K\log K+2K\log\frac{s_K(Y)-\epsilon_{1}}{\left\Vert X\right\Vert _{op}}.
\]

This implies 
\begin{align*}
2\log\kappa(W_{L:L_{1}+1}) & \leq\left\Vert \theta\right\Vert_2 ^{2}-LK+L_{1}K\log K-2K\log\frac{s_K(Y)-\epsilon_{1}}{\left\Vert X\right\Vert _{op}}\\
 & \leq c+L_{1}K\log K-2K\log\frac{s_K(Y)-\epsilon_{1}}{\left\Vert X\right\Vert _{op}}.
\end{align*}
\end{proof}

\globalisnc*
\begin{proof}
By assumption, there are parameters of the nonlinear part $\theta_{nonlin}$ such
that the representation at the end of the nonlinear layers already matches
the outputs $Z_{L_{1}}=Y$ with finite parameter norm $\left\Vert \theta_{nonlin}\right\Vert _2^{2}=c$.
We can now build a deeper network by simply setting all the linear
layers to be the $K$-dimensional identity, so that $\left\Vert \theta_{lin}\right\Vert_2 ^{2}=KL_2$,
leading to a total parameter norm of $\left\Vert \theta\right\Vert_2 ^{2}=c+KL_2$.
Since the outputs matches the labels exactly, we have that the
regularized loss is bounded by $\lambda(KL_2+c)/2$. This implies that any global minimizer $\theta^{*}$ satisfies 
\begin{align*}
\left\Vert Y-Z_{L}\right\Vert^2 _{F} & \leq\lambda(KL_2+c),\\
\left\Vert \theta^{*}\right\Vert_2 ^{2} & \leq KL_2+c.
\end{align*}
Choosing $\lambda\leq\frac{\epsilon_{1}^{2}}{KL_2 +c}$, we obtain
that the global minimizer $\theta^{*}$ must satisfy the assumptions
of Proposition \ref{prop:conditioning_from_interp+param_norm}, which readily gives the desired upper bound on $\kappa(W_{L:L_{1}+1})$.

To finish the proof, we show that all critical points of the regularized loss have balanced linear layers, so that we may relate the conditioning of $W_L$ to that of the product $W_{L:L_1+1}$. At any critical point, the gradient w.r.t. to any weight matrix amongst the linear layers must be zero, that is
\[
0=\nabla_{W_{\ell}}C_{\lambda}(\theta)=W_{L:\ell+1}^{\top}(Z_{L}-Y)Z_{L_{1}}^{\top}W_{\ell-1:L_{1}+1}^{\top}+\lambda W_{\ell}.
\]
This implies that $W_{\ell}=-\frac{1}{\lambda}W_{L:\ell+1}^{\top}(Z_{L}-Y)Z_{L_{1}}^{\top}W_{\ell-1:L_{1}+1}^{\top}$
for all linear layers $\ell$. Balancedness then follows directly
\[
W_{\ell}W_{\ell}^{\top}=-\frac{1}{\lambda}W_{L:\ell+1}^{\top}(Z_{L}-Y)Z_{L_{1}}^{\top}W_{\ell:L_{1}+1}^{\top}=W_{\ell+1}^{\top}W_{\ell+1}.
\]
Finally, the balancedness implies that
\[
\kappa(W_L)=\kappa(W_{L:L_{1}+1})^{\frac{1}{L_1}}\leq\left(\frac{\left\Vert X\right\Vert _{op}}{ s_K(Y)-\epsilon_{1}}\right)^{\frac{K}{L_1}}\exp\left(\frac{1}{2L_1}\left(c-L_{1}K+L_{1}K\log K\right)\right).
\]
\end{proof}

\largelearningrates*
\begin{proof}
As the NTK is a sum over all layers, we can lower bound it by the contribution of the linear layers only. Formally, we have 
\begin{align*}
    \left\Vert \nabla_{\theta}\mathrm{Tr}[Z_{L}A^\top]\right\Vert_2^{2}&\geq\left\Vert \nabla_{\theta_{lin}}\mathrm{Tr}[Z_{L}A^\top]\right\Vert_2 ^{2} \\
    &=\sum_{\ell=L_{1}+1}^{L}\left\Vert W_{\ell-1:L_{1}+1}Z_{L_{1}}A^\top W_{L:\ell+1}\right\Vert_F ^{2}
\end{align*}

Let $v_{1},\dots,v_{K}\in\mathbb{R}^{N}$ be $K$ orthonormal vectors
that span the preimage of $Z_{L}$ ($Z_{L}$ is rank $K$ as long
as $\epsilon_{1}\leq s_{K}(Y)$), and $e_{1},\dots,e_{K}\in\mathbb{R}^{K}$
be the standard basis of $\mathbb{R}^{K}$. We can then sum over $K^2$ possible choices of matrices $A=e_iv_j^T$ to obtain
\begin{align*}
K^2 \| \Theta \|_{op} \geq \sum_{i,j=1}^{K}\left\Vert \nabla_{\theta}(e_{i}^{T}Z_{L}v_{j})\right\Vert_2 ^{2} & \geq\sum_{i,j=1}^{K}\sum_{\ell=L_{1}+1}^{L}\left\Vert W_{\ell-1:L_{1}+1}Z_{L_{1}}v_{j}\right\Vert_2 ^{2}\left\Vert e_{i}^{T}W_{L:\ell+1}\right\Vert_2 ^{2}.
\end{align*}
We know use the fact that $P_{\mathrm{Im} W_{L:\ell}^T}W_{\ell-1:L_{1}+1}Z_{L_{1}} = W_{L:\ell}^+ W_{L:\ell}W_{\ell-1:L_{1}+1}Z_{L_{1}} = W_{L:\ell}^+ Z_{L}$ to obtain the lower bound
\begin{align*}
 K^2 \| \Theta \|_{op} & \geq\sum_{i,j=1}^{K}\sum_{\ell=L_{1}+1}^{L}\left\Vert W_{L:\ell}^{+}Z_{L}v_{j}\right\Vert_2 ^{2}\left\Vert e_{i}^{T}W_{L:\ell+1}\right\Vert_2 ^{2}\\
 & =\sum_{\ell=L_{1}+1}^{L}\left\Vert W_{L:\ell}^{+}Z_{L}\right\Vert _{F}^{2}\left\Vert W_{L:\ell+1}\right\Vert _{F}^{2}\\
 & \geq s_{K}(Z_{L})^{2}\sum_{\ell=L_{1}+1}^{L}\left\Vert W_{L:\ell}^{+}\right\Vert _{F}^{2}\frac{\left\Vert W_{L:\ell}\right\Vert _{F}^{2}}{\left\Vert W_{\ell}\right\Vert _{op}}\\
 & \geq\frac{\left(s_{K}(Y)-\epsilon_{1}\right)^{2}}{r^{2}}\sum_{\ell=L_{1}+1}^{L}\kappa(W_{L:\ell})^{2}.
\end{align*}
This then implies that 
\[
CL_{2}\geq\left\Vert \Theta\right\Vert _{op}\geq\frac{\left(s_{K}(Y)-\epsilon_{1}\right)^{2}}{K^{2}r^{2}}\sum_{\ell=L_{1}+1}^{L}\kappa(W_{L:\ell})^{2}.
\]
Let us now assume by contradiction that for all the layers $\ell\in \{L_{1}+1,\dots,L_{1}+M\}$
we have $\kappa(W_{L:\ell})^{2}>\frac{CL_{2}K^{2}r^{2}}{M\left(s_{K}(Y)-\epsilon_{1}\right)^{2}}$,
then 
\[
\frac{\left(s_{K}(Y)-\epsilon_{1}\right)^{2}}{K^{2}r^{2}}\sum_{\ell=L_{1}+1}^{L}\kappa(W_{L:\ell})^{2}>CL_{2},
\]
which yields a contradiction. Therefore, there must be
a layer within the layers $\ell\in\{L_{1}+1,\dots,L_{1}+M\}$ such that
$\kappa(W_{L:\ell})^{2}\leq\frac{CL_{2}K^{2}r^{2}}{M\left(s_{K}(Y)-\epsilon_{1}\right)^{2}}$.

Furthermore, since $\kappa(W_{L:\ell})\geq1$, we know that 
\[
\left\Vert \Theta\right\Vert _{op}\geq\frac{\left(s_{K}(Y)-\epsilon_{1}\right)^{2}}{K^{2}r^{2}}L_{2}.
\]
\end{proof}

\section{Technical Results}


\begin{lemma}
\label{lem:Lipschitzness_gradient_bounded_weights}Let $b\ge 1$ be s.t.\ $\|X_{:i}\|_2\le b$ for all $i$. Then, inside the set of parameters with bounded
weights $\left\Vert W_{\ell}\right\Vert _{op}\leq r_\ell$ for all $\ell\in [L]$ and $r_\ell\ge 1$, the gradient
of the loss $\nabla C_0(\theta)$ is $5N\beta b^{3}\left(\prod_{j=1}^{L}r_j\right)^3 L^{5/2}$-Lipschitz.
\end{lemma}

\begin{proof}
Consider two parameters $\theta=(W_\ell)_{\ell=1}^L, \theta'=(W_\ell')_{\ell=1}^L$, and let $Z_m, Z_m'$ be the corresponding outputs at layer $m$. Then, we obtain the following telescopic sum: 
\begin{equation}
\begin{split}    
(Z_m)_{:i}-(Z_m')_{:i}&=\sum_{\ell=1}^{L}(W_{m}\cdots W_{\ell+1}\circ\sigma\circ W_{\ell}\circ\sigma\circ W'_{\ell-1}\cdots\circ W'_{1})(X_{:i})\\
&-(W_{m}\cdots{}W_{\ell+1}\circ\sigma\circ W'_{\ell}\circ\sigma\circ W'_{\ell-1}\circ\cdots\circ W'_{1})(X_{:i}),
\end{split}    
\end{equation}
so that
\begin{equation}
\begin{split}    
\left\Vert (Z_m)_{i:}-(Z_m')_{i:}\right\Vert _{2}&\leq\sum_{\ell=1}^{m}\prod_{j=1}^{m-1} r_j\left\Vert W_{\ell}-W'_{\ell}\right\Vert _{F}\left\Vert X_{:i}\right\Vert_2 \\
&\leq b\prod_{j=1}^{L-1} r_j\sqrt{m}\sqrt{\sum_{\ell=1}^{m}\left\Vert W_{\ell}-W'_{\ell}\right\Vert _{F}^{2}}\leq b\prod_{j=1}^{L-1} r_j \sqrt{m}\left\Vert \theta-\theta'\right\Vert_2 .
\end{split}    
\end{equation}

Now, the gradient equals
\[
\nabla C_0(\theta)=\left(\sum_{i=1}^{N}\prod_{j=1}^{\ell-1}D_{j,i} W_{j} X_{:i}(Y_{:i}-(Z_L)_{:i})^{\top}\prod_{j=\ell+1}^L  W_j D_{j,i}\right)^\top_{\ell=1,\dots,L}
\]
where $D_{j, i}$ is a diagonal matrix with diagonal entries
given equal to the vector $\sigma'((Z_j)_{:i})$ for $j\in [L_1]$ and it is equal to the identity otherwise. Similarly, we define $D_{j, i}'$ as a diagonal matrix with diagonal entries
given equal to the vector $\sigma'((Z_j')_{:i})$ for $j\in [L_1]$ and equal to the identity otherwise. Since $\sigma'$
is $\beta$-Lipschitz, we have 
\[
\left\Vert D_{j,i}-D_{j,i}'\right\Vert _{F}\leq\beta b\prod_{j=1}^{L-1} r_j\sqrt{m}\left\Vert \theta-\theta'\right\Vert_2 .
\]
Furthermore since $\sigma'\leq1$ ,we have that $\left\Vert D_{j,i}\right\Vert _{op}\leq1$.
By summing over all terms that need to be changed, we obtain
\begin{align*}
\left\Vert \nabla C_0(\theta)-\nabla C(\theta')\right\Vert_2  & \leq\sum_{\ell=1}^{L}\sum_{i=1}^{N}\left\Vert X_{:i}\right\Vert_2 \left\Vert Y_{:i}-(Z_L)_{:i}\right\Vert_2 \sum_{m\neq\ell}\prod_{j\not \in \{m, L\}}r_j\left\Vert W_{m}-W'_{m}\right\Vert _{F}\\
 & +\sum_{\ell=1}^{L}\sum_{i=1}^{N}\left\Vert X_{:i}\right\Vert_2 \left\Vert Y_{:i}-(Z_L)_{:i}\right\Vert_2 \sum_{m=1}^{L_{1}}\left(\prod_{j=1}^{L-1}r_j\right)^2\beta \sqrt{m}\left\Vert \theta-\theta'\right\Vert_2 \left\Vert X_{:i}\right\Vert_2 \\
 & +\sum_{\ell=1}^{L}\sum_{i=1}^{N}\left\Vert X_{:i}\right\Vert_2 \left(\prod_{j=1}^{L-1}r_j\right)^2\sqrt{L}\left\Vert \theta-\theta'\right\Vert_2 \left\Vert X_{:i}\right\Vert_2, 
\end{align*}
where the three terms correspond to the effect of changing the $W_{\ell}$'s,
the $D_{\ell}$'s and the $Z_\ell$'s respectively. This can
then be simplified to 
\begin{align*}
\left\Vert \nabla C_0(\theta)-\nabla C_0(\theta)\right\Vert_2  & \leq Nb^{2}\left(1+\prod_{j=1}^{L}r_j\right)\prod_{j=1}^{L-1}r_j L^{3/2}\left\Vert \theta-\theta'\right\Vert_2 \\
 & +Nb^{3}\left(1+\prod_{j=1}^{L}r_j\right)\left(\prod_{j=1}^{L-1}r_j\right)^2 LL_{1}^{3/2}\beta\left\Vert \theta-\theta'\right\Vert_2 \\
 & +Nb^{2}\left(\prod_{j=1}^{L-1}r_j\right)^2 L^{3/2}\left\Vert \theta-\theta'\right\Vert_2 \\
 & \leq5N\beta b^{3}\left(\prod_{j=1}^{L}r_j\right)^3 L^{5/2}\left\Vert \theta-\theta'\right\Vert_2 ,
\end{align*}
where we use the fact that $\left\Vert (Z_L)_{:i}\right\Vert_2 \leq b\prod_{j=1}^{L}r_j$.
\end{proof}

\begin{lemma} \label{lem:approximate_balancedness_relate_weight_to_product}
    If the network satisfies
\begin{itemize}
\item approximate balancedness $\left\Vert W_{\ell+1}^{T}W_{\ell+1}-W_{\ell}W_{\ell}^{T}\right\Vert _{op}\leq\epsilon_{2}$ for $\ell\in\{L_1+1,\cdots L-1\}$,
\item bounded weights $\left\Vert W_{\ell}\right\Vert _{op}\leq r$ for $\ell\in\{L_1+1,\cdots, L\}$,
\end{itemize}
then we have $\left\|\left(W_LW_L^\top\right)^{L_2}-W_{L:L_1+1}W_{L:L_1+1}^\top\right\|_{op}\leq \frac{L_2^2}{2}\epsilon_2 r^{2(L_2-1)}$, where $L_2=L-L_1$.
\end{lemma}
\begin{proof}
    We denote $D_\ell:=W_\ell W_\ell^\top -W_{\ell+1}^\top W_{\ell+1}$ and $$E_\ell=\sum_{i=L_1+1}^\ell \left(W_{\ell+1}^\top W_{\ell+1}\right)^{i-L_1-1}D_\ell\left(W_\ell W_\ell^\top\right)^{\ell-i}.$$ We have that     \begin{align*}
        \left(W_\ell W_\ell^\top\right)^{\ell-L_1} &= D_\ell \left(W_\ell W_\ell^\top\right)^{\ell-L_1-1}+W_{\ell+1}^\top W_{\ell+1}\left(W_\ell W_\ell^\top\right)^{\ell-L_1-1}\\
        &= D_\ell\left(W_\ell W_\ell^\top\right)^{\ell-L_1-1}+W_{\ell+1}^\top W_{\ell+1} D_\ell\left(W_\ell W_\ell^\top\right)^{\ell-L_1-2}+  \left(W_{\ell+1}^\top W_{\ell+1}\right)^2\left(W_\ell W_\ell^\top\right)^{\ell-L_1-2}\\
        &=\dots\\
        &=\left(W_{\ell+1}^\top W_{\ell+1} \right)^{\ell-L_1}+\sum_{i=L_1+1}^\ell \left(W_{\ell+1}^\top W_{\ell+1}\right)^{i-L_1-1}D_\ell\left(W_\ell W_\ell^\top\right)^{\ell-i}\\
        &= \left(W_{\ell+1}^\top W_{\ell+1} \right)^{\ell-L_1}+E_\ell.
    \end{align*}
    Then,
    \begin{align*}
        W_{L:L_1+1} W_{L:L_1+1}^\top&=W_L\cdots W_{L_1+1} W_{L_1+1}^\top\cdots W_L^\top\\
        &= W_L\cdots W_{L_1+2}\left(W_{L_1+2}^\top W_{L_1+2}+E_{L_1+1}\right) W_{L_1+2}^\top\cdots W_L^\top\\
        &= W_L\cdots W_{L_1+3}\left(W_{L_1+2}^\top W_{L_1+2}\right)^2W_{L_1+3}^\top\cdots W_L^\top+\mathcal{E}_{L_1+1}\\
        &=\dots\\
        &=\left(W_LW_L^\top\right)^{L_2}+\sum_{i=L_1+1}^{L-1} \mathcal{E}_i,
    \end{align*}
    where $\mathcal{E}_i=W_L\cdots W_{i+1}E_iW_{i+1}^\top \cdots W_L^\top$. As a result, we have $$\left\|\left(W_LW_L^\top\right)^{L_2}-W_{L:L_1+1}W_{L:L_1+1}^\top\right\|_{op}=\left\| \sum_{i=L_1+1}^{L-1}\mathcal{E}_i\right\|_{op}\le\sum_{i=L_1+1}^{L-1}\left\|\mathcal{E}_i\right\|_{op}.$$
    Since $\|W_k\|_{op}\le r$ for all $k=\{L_1+1,\cdots L\}$, we have $\|E_\ell\|_{op}\le (\ell-L_1)\epsilon_2 r^{2(\ell-L_1-1)}$ and $\|\mathcal{E}_\ell\|_{op}\le r^{2(L-\ell)}(\ell-L_1)\epsilon_2 r^{2(\ell-L_1-1)}=(\ell-L_1) \epsilon_2 r^{2(L_2-1)}.$ Thus, we have
    $$\left\|\left(W_LW_L^\top\right)^{L_2}-W_{L:L_1+1}W_{L:L_1+1}^\top\right\|_{op}\le \frac{L_2^2}{2}\epsilon_2 r^{2(L_2-1)}.$$
\end{proof}

\end{document}

%% file: math_commands.tex
\usepackage{amsmath,amsfonts,bm}

\def\1{\bm{1}}

\DeclareMathAlphabet{\mathsfit}{\encodingdefault}{\sfdefault}{m}{sl}
\SetMathAlphabet{\mathsfit}{bold}{\encodingdefault}{\sfdefault}{bx}{n}

